\documentclass[runningheads]{llncs}
\usepackage[T1]{fontenc}
\usepackage{graphicx}
\usepackage{CJKutf8}
\usepackage{bm}
\usepackage{amsmath}
\usepackage[ruled,vlined,linesnumbered]{algorithm2e}
\usepackage{comment}
\usepackage{amssymb}
\usepackage{caption}
\usepackage{subfigure}
\usepackage{booktabs}
\begin{document}
\title{An Efficient Data Analysis Method for Big Data using Multiple-Model Linear Regression\thanks{This work was supported by the National Natural Science Foundation of China under grants 61832003, Shenzhen Science and Technology Program (JCYJ202208181002205012) and Shenzhen Key Laboratory of Intelligent Bioinformatics (ZDSYS20220422103800001).}}
%
%
\author{Bohan Lyu\inst{1,2}\orcidID{0009-0005-6462-8942} \and
Jianzhong Li\inst{1,2}\orcidID{0000-0002-4119-0571}}
\authorrunning{Bohan Lyu, Jianzhong Li.}
%

\institute{
	\email{18b903024@stu.hit.edu.cn}\\
	\email{lijzh@hit.edu.cn}
	\\
	\inst{1}Harbin Institute of Technology, Harbin, Heilongjiang, China
	\\
	\inst{2}Shenzhen Institute of Advanced Technology, Chinese Academy of Sciences, Shenzhen, China
}

\maketitle              
\begin{abstract}
This paper introduces a new data analysis method for big data using a newly defined regression model named multiple model linear regression(MMLR), which separates input datasets into subsets and construct local linear regression models of them. The proposed data analysis method is shown to be more efficient and flexible than other regression based methods. This paper also proposes an approximate algorithm to construct MMLR models based on $(\epsilon,\delta)$-estimator, and gives mathematical proofs of the correctness and efficiency of MMLR algorithm, of which the time complexity is linear with respect to the size of input datasets. This paper also empirically implements the method on both synthetic and real-world datasets, the algorithm shows to have comparable performance to existing regression methods in many cases, while it takes almost the shortest time to provide a high prediction accuracy.

\end{abstract}
\section{Introduction}
Data analysis plays an important role in various aspects, because it tells the features of data and helps predicting tasks.
Regression with parametric models, especially linear models, is a typical data analysis method.

Let $\bm{DS}=\{(y, x_1 , x_2 , \cdots , x_k)\}$ be a $k+1$ dimensional dataset with $n$ elements, 
where $y$ is called as response variable, 
$x_i$ is called as explanatory variable for reach $1 \leq i \leq k$. 
The task of regression is to determine a function $\hat{y} = f(x_1,x_2,\cdots,x_k)$ using $\bm{DS}$, minimizing $ \mathbb{E}(y - \hat{y})^2$.
As for linear regression, $f(x_1,\cdots,x_k)$ is a linear function of $x_i$s. And there's the assumption that $y = \beta_0 + \beta_1 x_1 + \beta_2 x_2 + \cdots + \beta_k x_k + \varepsilon$ for each $(y, x_1 , x_2 , \cdots , x_k) \in \bm{DS}$,
where $\varepsilon$ is a random noise obeying normal distribution,
and $\beta_i$ ($1 \leq i \leq k$) are constants.

Under such assumptions, linear regression model has statistical advantages and high interpretability.
The numeric value of parameters can show the importance of variables, and the belonging information about the confidence coefficients and intervals make the model more credible in practice\cite{book2021lra}.
Therefore, linear regression is widely used in research areas requiring high interpretability, such as financial prediction, investment forecasting, biological and medical modelling, etc.
Most machine learning and deep learning models might be more precise in predicting tasks, but the black-box feature limits their ranges of application.

However, linear regression still faces challenges in case of big data. Because big data has a feature that different subsets of a dataset fitting highly different regression models, which is described as \emph{diverse predictor-response variable relationships}(DPRVR) in \cite{dong2015pattern}.
An example of real-world data is TBI dataset in \cite{TBI2005}, which is used to predict traumatic brain injury patients' response with sixteen explanatory variables.
The \emph{root mean square error} ($RMSE$) is 10.45 when one linear regression model is used to model the whole TBI.
While TBI is divided into 7 subsets and 7 different linear regression models are used individually, the $RMSE$ is reduced to 3.51.
TBI shows that the DPRVR commonly appears in real-world data and big datasets.
This feature indicates that it is much better to use multiple linear models rather than only one to model a big dataset.

Nevertheless, there is no efficient multi-model based regression algorithms for big datasets till now 
since the time complexities of existing multi-model based regression algorithms are too high to model big datasets. Piecewise Linear Regression, or segmented regression\cite{siahkamari2020piecewise,arumugam2003emprr,wang1996mtree,lokshtanov2021ppoly,bemporad2022piecewise} are the only kinds so far. They divide the input datasets into several connected areas, and then construct local models using data points in each connected area, which is similar to the example in Figure 1 from Appendix A.
These kinds of regression models has high prediction accuracy because it considers DPRVR, and has high interpretability since the local models could have explicit expressions.

But there are still shortcomings of the existing multiple models based regression methods, which are shown as follows.

1. The time complexity of the methods is high. The state-of-art algorithm, PLR, has the time complexity of $O(k^2 n ^5)$ \cite{siahkamari2020piecewise}.

2. The subsets being used to construct multiple regression models must be hyper cubes\cite{diakonikolas2020sr} or generated by partition the given dataset by hyperplanes\cite{lokshtanov2021ppoly}. Thus, the accuracy of the methods is lower when the subsets of a given dataset are not hyper cubes or can not be generated by hyperplanes. 

3. Some methods need apriori knowledge that is difficult to get\cite{arumugam2003emprr}.

To overcome the three disadvantages above, this paper proposes a new multi-model based linear regression method named as MMLR. 

Specifically, MMLR algorithm is outlined in \textbf{Algorithm 1}. Noticing that every $d\in \bm{DS}$ has the form $d=(y,x_1,x_2,\cdots,x_k)$ and the output of the MMLR algorithm is $\bm{M}=\{(f_i,S_i)\}$, where $S_i$ is a subset of $\bm{DS}$ and $f_i$ is a linear regression model fits $S_i$.

\begin{algorithm}
	\label{a_approximation}
	\caption{MMLR}
	\KwIn{ $(k+1)$-dimensional Dataset $\bm{DS}$ with $n$ data points;}
	\KwOut{regression model $\bm{M}=\{(f_i,S_i)\}$}
	
	$i=0$,$\bm{M}=\emptyset$, $\bm{WS}=\bm{DS}$\;
	\While{$\bm{WS} \neq \emptyset$}{
		$i=i+1$ \;
		Select a hypercube $H_i$ with small size in $\bm{WS}$\;
		Construct a linear regression model $f_i$ of $H_i$ \;
		Compute the error bound $eb_i$ using $f_i$ to predict $H_i$ \;
		$S_i = H_i$ \;
		\For{each $(y, x_1, ..., x_k)$ in $\bm{WS}-H_i$}{ 
			\If{$|f_i(x_1, ..., x_k)-y|<eb_i$}{
				$S_i$=$S_i \cup (y, x_1, ..., x_k)$\;}
		}
		$\bm{WS}=\bm{WS}-S_i$, $\bm{M} = \bm{M} \cup \{(f_i,S_i)\}$ \;
	}
	
	\Return $\bm{M}$
\end{algorithm}

This paper has proved that the time for constructing every $f_i$ is $O(k^2/\varepsilon^2)$, where $\varepsilon$ is a user-given upper bound of the max error of all $f_i$'s parameters. 
It's been proved that the time cost of MMLR algorithm is $O(m(n + (k / \varepsilon)^2+k^3))$ in Section 4, where $m$ is the number of models. The time complexity of MMLR is much lower than $O(k^2 n^5)$, since $(k / \varepsilon)^2$ is far more less than $n^5$. Therefore, the disadvantage 1 above is overcome.

From steps 4-10 of the MMLR algorithm, every subset $S_i$ can be any shape rather than a hypercube or a subset generated by partitioning $\bm{DS}$ using only hyper-planes. Thus, the disadvantage 2 above is overcome also.

The MMLR algorithm iteratively increase the number of regression models so that it could always find the suitable number of models to optimize the prediction accuracy without knowing the number of models as apriori knowledge. In fact, MMLR only take $\bm{DS}$ as input, which overcomes the disadvantage 3 above.

The major contributions of this paper are as follows.  
\begin{itemize}
	\item The problem of constructing the optimized multiple linear regression models for a given dataset is formally defined and analyzed. 
	\item A heuristic MMLR algorithm is designed for solving the problem above. 
	MMLR algorithm can deal with the DPRVR of big datasets,
	and overcomes the disadvantages of the existing multi-model based linear regression methods.
	\item The time complexity of MMLR algorithm is analyzed, which is lower than the state-of-art algorithm. 
	The accuracy of MMLR algorithm and related mathematical conclusions are proved.
\end{itemize}

The rest of this paper is organized as follows. Section 2 gives the formal definition of the problem. Section 3 proves the necessary mathematical theorems. Section 4 gives the design and analysis of the algorithm. Finally, section 5 concludes the paper.

\section{Preliminaries and Problem Definition}
\subsection{Regression and Linear Regression}

The definition of traditional linear regression problem is given as follows.
\begin{definition}[Linear Regression Problem]
	$ $
	
	\textbf{Input:} A numerical dataset $\bm{DS}=\{(\bm{x}_i, $ $y_i)\ | 1 \leq i \leq n \}$, where $\bm{x}_i \in \mathbb{R}^k, y_i \in \mathbb{R}$, $y_i = f(\bm{x}_i) + \varepsilon_i$ for some function $f$, $\varepsilon_i \sim N(0,\sigma^2)$, and all $\varepsilon_i$s are independent.
	
	\textbf{Output:} A function $\hat{f}(\bm{x})=\bm{\beta} \cdot (1,\bm{x})=\beta_0 + \beta_1 x_1 + \cdots + \beta_k x_k$ such that $\mathbb{E}[(y-\hat{f}(\bm{x}))^2]$ is minimized for $\forall (\bm{x},y) \in \bm{DS}$. 
\end{definition}

Douglas C. Montgomery, Elizabeth A. Peck and G. Geoffrey Vining have proved that minimize the $\mathbb{E}[(y-\hat{f}(\bm{x}))^2]$ is equivalent to minimizing $\sum(y_i - \hat{f}(\bm{x}_i))^2$ on $\bm{DS}$ in the case of linear regression \cite{book2021lra}. Besides, it's trivial that a $k$-dimensional linear function can perfectly fit any $n$ data points when $n \leq k+1$. It is said $\bm{DS}$ to be centralized, if using $x_{ij} - \overline{x}_j$ to substitute $x_{ij}$, using $y_i - \overline{y}$ to substitute $y_i$, for $i=1,2,\cdots,n,j=1,2,\cdots,k$, where $\overline{x}_j = (x_{1j} + \cdots + x_{nj})/n,\overline{y}=(y_1 + \cdots + y_n)/n$. Obviously, when $\bm{DS}$ is centralized, there's always $\beta_0=0$ so that $\bm{\beta}=(\beta_1,\cdots,\beta_k)$. Therefore, this paper always assumes that $n > k+1$, $\bm{DS}$ is centralized and all $\bm{x}_i$ are i.i.d. and uniformly drawn from the value domain of $\bm{x}$ for convenience of analysis.

The simplest method to construct linear regression model is \emph{pseudo-inverse matrix method}.
It transforms the given $\bm{DS}$ into a vector $\bm{y} = (y_1, y_2,\cdots,y_n)$ and an $n \times (k+1)$ matrix $\bm{X}$:$$\bm{X}=\begin{pmatrix}
	1 & x_{11} & x_{12} & \cdots & x_{1k} \\
	1 & x_{21} & x_{22} & \cdots & x_{2k} \\
	\vdots & \vdots & \vdots & \vdots \\
	1 & x_{n1} & x_{n2} & \cdots & x_{nk} \\
\end{pmatrix} .$$ 
The $x_{ij}$ in $\bm{X}$ equals to the value of the $i$-th data point's $j$-th dimension in $\bm{DS}$. It is called the \emph{data matrix} of $\bm{DS}$. Then, using the formula $\hat{\bm{\beta}}=(\bm{X'}\bm{X})^{-1}\bm{X'}\bm{y}$, the linear regression model $\hat{f}$ could be constructed.
The time complexity of pseudo-inverse matrix method is $O(k^2 n + k^3)$ \cite{book2021lra}. When $k$ is big enough, gradient methods is more efficient than pseudo-inverse matrix method. Generally, every method's complexity has the bound $O(k^2 n + k^3)$, so this paper use it as the time complexity of linear regression in common.

To judge the goodness of a linear regression model $\hat{f}$, the p-value of F-test is a convincing criterion, which denotes as $p_F(\hat{f})$ in this paper. Generally, linear regression model $\hat{f}\cong f$ when $p_F(f)<0.05$, and $p_F(f)$ can be calculated in $O(n)$ time\cite{book2021lra}.

Besides, this paper uses the $(\epsilon , \delta)-estimator$ to analyze the performance of linear regression models constructing by subsets of $\bm{DS}$. The formal definition of it is as follows.

\begin{definition}[$(\epsilon , \delta)-estimator$]
	
	$\hat{I}$ is an $(\epsilon , \delta)-estimator$ of $I$ if $ \mathrm{Pr}\{ |\hat{I}-I|\geq \epsilon \} \leq \delta$ for any $\epsilon \geq 0$ and $0 \leq \delta \leq 1$, where $\mathrm{Pr}(X)$ is the probability of random event $X$.
	
\end{definition}

Intuitively, $(\epsilon , \delta)-estimator$ $\hat{I}$ of $I$ has high possibility to be very close to $I$'s real value. By controlling some parameters for the calculation of $\hat{I}$, one can get an $I$'s arbitrarily precise estimator.

\subsection{Multiple-model Linear Regression}
In rest of the paper, the multiple-model linear regression is denoted as MMLR.
The MMLR problem is defined as follows.

\begin{definition}[Optimal MMLR problem]
	$ $
	
	\textbf{Input:}A numerical dataset $\bm{DS}=\{(\bm{x}_i, $ $y_i)\ | 1 \leq i \leq n \}$, where $\bm{x}_i \in \mathbb{R}^k$, $y_i \in \mathbb{R}$, $y_i = f(\bm{x}_i) + \varepsilon_i$ for a function $f$, $\varepsilon_i \sim N(0, \sigma^2)$ and all $\varepsilon_i$s are independent, maximal model number $M_0$, smallest volume $n_0$.
	
	\textbf{Output:}$\bm{M}_{opt} = \{(\hat{g}_i,S_i) | 1 \leq i \leq m\}$ such that $MSE(\bm{M})=\sum\limits_{1 \leq i \leq m} \sum\limits_{(\bm{x}_{j},y_j) \in S_i}$ $(y_{j} - \hat{g}_i(\bm{x}_{j}))^2$ is minimized, where $\hat{g}_i$ is a linear regressoin models of $S_i \subseteq \bm{DS}$, $S_1 \cup S_2 \cup \cdots \cup S_m = \bm{DS}$, $S_i \cap S_j = \emptyset$ for $i \neq k$, and $|S_i| \geq n_0, m \leq M_0$.
	
\end{definition}

The optimal MMLR problem is expensive to solve since it has a similar to Piecewise Linear Regression problem \cite{siahkamari2020piecewise}. By far, the best algorithm to solve piecewise linear regression problem without giving the number of pieces beforehand is $O(k^2 n^5)$ \cite{siahkamari2020piecewise}. 
Therefore, this paper focuses on the approximate optimal solution of MMLR problem. To bound the error of a linear function, this paper set $|f_1-f_2|=\Vert \bm{\beta}_1 - \bm{\beta}_2 \Vert_\infty=\max\limits_{1\leq i \leq k} |b_{1i}-b_{2i}|$, where $f_j(\bm{x}) = \bm{\beta}_j \bm{x} = b_{j1}x_1 + \cdots b_{jk}x_k,j=1,2$. The definition of Approximately Optimal MMLR problem is as follows.

\begin{definition}[Approximately Optimal MMLR problem]
	$ $
	
	\textbf{Input:}A numerical dataset $\bm{DS}=\{(\bm{x}_i, $ $y_i)\ | 1 \leq i \leq n \}$, where $\bm{x}_i \in \mathbb{R}^k$, $y_i \in \mathbb{R}$, $y_i = f(\bm{x}_i) + \varepsilon_i$ for a function $f$, $\varepsilon_i \sim N(0, \sigma^2)$ and all $\varepsilon_i$s are independent, $\epsilon > 0, 0 \leq \delta \leq 1$, maximal model number $M_0$, smallest volume $n_0$.
	
	\textbf{Output:}$\bm{M} = \{(\hat{f}_i,S_i) | 1 \leq i \leq m\}$ such that $ \mathrm{Pr} \{ \max\limits_i |\hat{f}_i - \hat{g}_i| \geq \epsilon \} \leq \delta$, $\hat{f}_i$ is a linear regression models of $S_i \subseteq \bm{DS}$, $S_1 \cup S_2 \cup \cdots \cup S_m = \bm{DS}$ and $S_i \cap S_j = \emptyset$ for $i \neq k$, and $|S_i| \geq n_0, m \leq M_0$, $\bm{M}_{opt} = \{(\hat{g}_i,S_i) | 1 \leq i \leq m\}$.
	
\end{definition}

In the end of this section, necessary denotations are given as follows.
Intuitively, $\bm{DS}$ is the input dataset and $|\bm{DS}|=n$, $(\bm{x}_i,y_i) \in \bm{DS}$ denotes the $i$-th data points, $i= 1,2,\cdots,n$, $\bm{x}_i=(x_{i1},x_{i2},\cdots,$ $x_{ik}) \in \mathbb{R}^k, y_i \in \mathbb{R}$. The $j$-th linear function is denoted as $f_j(\bm{x}) = \beta_{j0} + \beta_{j1} x_1 + \cdots + \beta_{jk} x_k$, and $\bm{\beta}_j = (\beta_{j0},\beta_{j1},\cdots,\beta_{jk})$ is the coefficient vector. $\bm{X}_{n\times (k+1)}$ is the data matrix of $\bm{DS}$, of which $\bm{X}_{ij} = x_{ij}$. Finally, every estimator of a value, vector or function $I$ is $\hat{I}$.

\section{Mathematical Foundations}
This section gives proofs of the necessary mathematical theorems used for the MMLR algorithm. We first discuss the settings and an existed mathematical result used in this section.

Suppose that $\bm{D}$ is a centralized dataset of size $n$, it's reasonable to set $y_i = f(\bm{x}_i) = \bm{\beta}\bm{x}= \beta_1 x_{i1} + \cdots + \beta_k x_{ik} + \varepsilon_i$ for each $(x_{i1}, x_{i2}, \cdots , x_{ik}$ $,y_i) \in \bm{D}$, $\varepsilon_i \sim N(0,\sigma^2)$, from the discussion in section 2.1. Besides, $\hat{f}(\bm{x})=\hat{\bm{\beta}}\bm{x}$ is the linear regression model of $\bm{D}$ constructed by \emph{least square criterion}. Thus there's the result in the following \emph{Lemma 1}.

\begin{lemma}\label{lemma1}
	
	Let $\bm{X}_{n\times k}$ be the data matrix of $\bm{D}$,$n \geq k+2$,
	and $\bm{\hat{\beta}}$ be the least square estimator of $\bm{\beta}$, then $\bm{\hat{\beta}}$ is unbiased, and
	the covariance matrix of $\hat{\bm{\beta}}$ is: $$\mathrm{Var}(\hat{\bm{\beta}})= \sigma^2 (\bm{X}'\bm{X})^{-1}.$$
	
\end{lemma}

\subsection{Theorems related to Sampling}

MMLR separates an input dataset into several disjoint subsets and construct local models for them. This subsection discusses how to construct local linear regression model on one subset efficiently.


The main process of constructing $\hat{f}$ using subset $PS \subseteq \bm{D}$, \emph{Cons}-$\hat{f}$ for short, is given as follows.

Step 1: Independently sample $PS_1, PS_2, \cdots, PS_t$ without replacement from $\bm{D}$, where $|PS_i| = p$ for $1 \leq i \leq t$.

Step 2: Use least square method to construct linear regression model $\hat{f}^{(i)} = \sum_{j=1}^k \hat{\bm{\beta}}^{(i)} \bm{x}$ for each $PS_i$.

Step 3: Let $\hat{\bm{\beta}}$ be the average of all $\hat{\bm{\beta}}^{(i)}$, where $$\hat{\beta}_j = \frac{1}{t} \sum_{i=1}^t \hat{\beta}_j^{(i)}, j=1,\cdots,k.$$

Let $PS = PS_1 \cup PS_2 \cup \cdots \cup PS_t$ and $|PS| = n_s$, then $\hat{f}$ is a linear regression model constructed from $PS$. 
\emph{Theorem 1} given in this section shows that $\hat{f}$ satisfies
\begin{equation}
	\mathrm{Pr}\Big\{\max \limits_j  \lvert \hat{\beta}_j -\beta_j \rvert \ge \epsilon \Big\} \le \delta
\end{equation}
for given $\epsilon \geq 0$ and $0 \leq \delta \leq 1$ if $|PS|$ is big enough.

By the steps of \emph{Cons}-$\hat{f}$, all $\hat{\bm{\beta}}^{(i)}$ obey the same distribution since every $PS_i$ is sampled from $\bm{D}$ by the same way independently.
Noticing that $\hat{\bm{\beta}}_j^{(i)}$ is a least-square estimator of $\bm{\beta}$ constructed using $PS_i$.
\emph{Lemma 1} shows that $\hat{\bm{\beta}}$ satisfies $\mathbb{E}(\hat{\bm{\beta}}) = \bm{\beta}$ and $\mathrm{Var}(\hat{\bm{\beta}}) = \sigma^2(\bm{X}'\bm{X})^{-1}$ when $n \geq k+2$.
So let $p = k+2$ and $n_s = pt = (k+2)t$, the minimum $n_s$ such that $\mathrm{Pr}\{ \max |\hat\beta_j - \beta_j| \geq \epsilon \} \leq \delta$ only depends on $t$.
Furthermore, we can give the following \emph{Lemma 2}, and then prove \emph{Theorem 1}.

\begin{lemma}\label{lemma1}
	
	Letting $\hat{\bm{\beta}}=(\hat{\beta}_1,\cdots,\hat{\beta}_j)$ be constructed by the procedure of Cons-$\hat{f}$, and $\bm{X}_i$ be the data matrix of the $PS_i$, for $1 \leq i \leq t$,
	then $\hat{\bm{\beta}}$ is an unbiased estimator of $\bm{\beta}$ and the covariance matrix of $\hat{\bm{\beta}}$ is:$$\mathrm{Var}(\hat{\bm{\beta}})= \frac{\sigma^2}{t^2} \sum_{i=1}^t (\bm{X}_i' \bm{X}_i)^{-1}.$$
	
\end{lemma}

\begin{proof}
	By the linearity of mathematical expectation and sum of mutually independent variables' variance, the conclusion of this lemma is obvious.
\qed\end{proof}

\begin{theorem}
	Let $\hat{\bm{\beta}}=(\hat{\beta}_0,\hat{\beta}_1,\cdots,\hat{\beta}_j)$ be constructed by the procedure of Cons-$\hat{f}$. If $\Phi(\frac{\epsilon \sqrt{t}}{\nu}) \geq \frac{2-\delta}{2}$, then $$\mathrm{Pr}\Big\{\max \limits_j  \lvert \hat{\beta}_j -\beta_j \rvert \ge \epsilon \Big\} \le \delta ,$$for any $\epsilon >0$ and $0 \leq \delta \leq 1$, where $\Phi(x)$ is the distribution function of standard normal distribution, and $\nu$ is the biggest standard deviation of all $\hat{\beta}_j$s.
\end{theorem}
\begin{proof}
	The proof of \emph{Theorem 1} is shown in Appendix B.
\qed\end{proof}

\emph{Theorem 1} can be used to decide the needed $n_s$ to satisfy inequality (1) by the following steps.

1. Check the table of normal distribution function for $\Phi(\frac{\epsilon \sqrt{t}}{\nu}) \geq \frac{2- \delta}{2}$, and $\nu = \max \nu_j = \max \frac{\sigma e_{jj}}{t} = \frac{\sigma \max e_{jj}}{t}$, get $$t > (\frac{1}{\epsilon}\Phi^{-1}(\frac{2-\delta}{2})\sigma \max \sqrt{e_{jj}})^{2/3} ;$$

2. Let $t > (\frac{1}{\epsilon}\Phi^{-1}(\frac{2-\delta}{2})\sigma)^{2/3}$ since $e_{jj}$ is always far more less than 1;

3. Let $n_s = p (\frac{1}{\epsilon}\Phi^{-1}(\frac{2-\delta}{2})\sigma)^{2/3}$.

Besides, it doesn't necessary to carry out the three steps of \emph{cons}-$\hat{f}$ in practice. In fact, we can directly construct $\hat{f}$ on the sample $PS$ by using least square method to satisfy inequation (1).
The following \emph{Theorem 2} shows the correctness.

\begin{theorem}
	Suppose $PS,t,n_s,\hat{\bm{\beta}} = (\hat{\beta}_1, \cdots, \hat{\beta}_k)$ are defined as in previous part, $\tilde{\bm{\beta}} = (\tilde{\beta}_1, \cdots, \tilde{\beta}_k)$ is the least square estimator of $PS$'s $\bm{\beta}$, then if $\hat{\bm{\beta}}$ satisfies inequality \emph{(1)}, $\tilde{\bm{\beta}}$ also satisfies inequality \emph{(1)}.
\end{theorem}
\begin{proof}
	The proof of \emph{Theorem 2} is shown in Appendix C.
\qed\end{proof}

From \emph{Theorem 2}, the \emph{Cons}-$\hat{f}$ can be simplified to the following \emph{Cons}-$\hat{f}$-\emph{New}, which is used in \emph{Algorithm 3} in Section 4. 

Step 1: Sample $PS$ from $\bm{D}$, where $|PS| = n_s$.

Step 2: Use least square method to construct the parameters of linear regression model $\hat{\bm{\beta}} = (\hat{\beta}_1,\hat{\beta}_2, \cdots, \hat{\beta}_k)$.

Finally, the following \emph{Theorem 3} shows that $n_s$ is not necessary to be large.

\begin{theorem}
	There exists an $n_s = O(\frac{1}{\epsilon^2})$ such that $$\mathrm{Pr}\Big\{\max \limits_j  \lvert \hat{\beta}_j -\beta_j \rvert \ge \epsilon \Big\} < 10^{-6}.$$
\end{theorem}
\begin{proof}
	The proof of \emph{Theorem 3} is shown in Appendix D.
\qed\end{proof}

\subsection{Theorem related to the Measures of Subsets}
By \emph{Algorithm 1}, MMLR constructs $\hat{f}_i$ using subsets $H_i \cap \bm{DS} \subseteq \bm{DS}$ initially, where $H_i$ is a hypercube whose edges are parallel to coordinate axis of $\mathbb{R}^k$. Subsection 3.1 shows that using a subset $PS$ randomly sampled from $H_i \cap \bm{DS}$ to construct $\hat{f}_i$ can be very accurate. However, the measures of $H_i$ is also a key factor influencing the accuracy of $\hat{f}_i$. Intuitively, the larger $H_i$ is, the more accurate $\hat{f}_i$ is. However, the $H_i$ is not necessary to be very large when $\epsilon$ and $\delta$ are given. This subsection discusses the necessary measure of $H_i$ to satisfy inequality (1). The following subsection has the same mathematical assumptions as section 3.1 . 

Limited by the length of the paper, we can only show the following necessary conclusions, the proofs of them and other intermediate results \emph{Lemma 3.1-3.4} are shown in Appendix E.

\begin{lemma}
	Given $\epsilon > 0, 0 < \delta <1$, then for any $\epsilon' >  0$, there exists an $n_0$ such that when $L \geq \frac{4 \sqrt{3} \sigma}{\epsilon \sqrt{n \delta}}$ and $n > n_0$, $\mathrm{Pr}\{ |\hat{\beta}_j - \beta_j| \geq \epsilon \} \leq \delta$ holds with possibility no less than $1-\epsilon'$. Further, $\mathbb{E}(|\hat{\beta}_j - \beta_j|)$ is monotonically decreasing at $nL^2$.
\end{lemma}

The Lemmas in this section show that for any $j=1,2,\cdots,k$, the value range of the $j$-th dimension of $\bm{D}$ influences the error of $\hat{\beta}_j$, and $\mathrm{Pr}\{ |\hat{\beta}_j - \beta_j| \geq \epsilon \}$ is in inverse proportion to $L^2$. It means that when $n$ is fixed, one can sample from larger value range of $\bm{D}$ to get more precise $\hat{\beta}_j$. 

Besides, the sample size $n$ has no need to be very big. The $\mathrm{Pr}\{ \bm{x}' \bm{A} \bm{x} \geq \frac{1}{2} \bm{x}' \bm{x} \}$ in \emph{Lemma 3.3} has a very fast convergence speed. When $n > 3k$, $\mathrm{Pr}\{ \bm{x}' \bm{A} \bm{x} \geq \frac{1}{2} \bm{x}' \bm{x} \}$ has already larger than 0.99. From \emph{Lemma 3.2} we could know that when $n > 65$, $\mathrm{Pr}\{\bm{x}'\bm{x} \geq \frac{nL^2}{24}\} \geq 0.95$. Since $\mathrm{Pr}\{ |\hat{\beta}_j - \beta_j| \geq \epsilon \} \leq \delta$ requires both $\bm{x}'\bm{x} \geq \frac{nL^2}{24}$ and $\bm{x}' \bm{A} \bm{x} \geq \frac{1}{2} \bm{x}' \bm{x}$, the inequality holds with possibility larger than $0.95$ when $n > \max\{65,3k\}$.

According to the discuss above, we propose the following process to construct $\hat{f}$ using subset $HS \subset \bm{D}$ with smallest $|HS|$, $Cons-H\hat{f}$ for short.

$Cons-H\hat{f}$:

Step 1: Given $n_s>\max\{65,3k\}$, calculate $L = \frac{4\sqrt{3}\sigma}{ \epsilon\sqrt{n_s \delta}}$.

Step 2: Randomly choose a data point $\bm{d} \in \bm{D}$, which satisfies $\min x_i + \frac{L}{2} \leq d_i \leq \max x_i - \frac{L}{2}$ for $i=1,2,\cdots,k$. Let $H = [d_1 - \frac{L}{2} , d_1 + \frac{L}{2}] \times \cdots \times [d_k - \frac{L}{2} , d_k + \frac{L}{2}]$

Step 3: Let $HS = \bm{D} \cap H$. If $|HS|\geq n_s$, use least square method to construct $\hat{f}$ on $HS$; else, increase $H$ till $|HS|\geq n_s$ or $|HS|=|\bm{D}|$, use least square method to construct $\hat{f}$ on $HS$.

In conclusion, the following \emph{Theorem 4} is correct by the discussion above.

\begin{theorem}
	If $|\bm{DS}|\geq n_s$ and $\max x_i - \min x_i \geq L$ for $i=1,\cdots,k$, the least square estimator $\hat{f}$ constructing by $Cons-H\hat{f}$ satisfies $\mathrm{Pr}\Big\{\max \limits_j  \lvert \hat{\beta}_j -\beta_j \rvert \ge \epsilon \Big\} \le \delta$.
\end{theorem}

\section{Algorithm and Analysis}

This section shows the pseudo-code of MMLR algorithm, as well as the details of illustrations and analysis.

\subsection{Algorithm}
Firstly, the general idea of MMLR has been shown in \emph{Section 1}. The detail of MMLR algorithm is shown in \emph{Algorithm 2}, the invoked algorithm \emph{Subset} is shown in \emph{Algorithm 3}. Specifically, MMLR uses pre-processing, pre-modelling, examine, grouping these four important phases to iteratively solving the problem. 

\begin{algorithm}[!htb]
	\label{a_approximation}
	\SetKwBlock{DoWhile}{Do}{end}
	\caption{MMLR($\bm{DS}, \epsilon, \delta, M_0, N_0$)}
	\KwIn{A $k$-dimensional dataset $\bm{DS}$ with $N$ data points, error bounds $\epsilon > 0, 0 < \delta < 1$, least subset volume $M_0$ and largest model number $N_0$;}
	\KwOut{$\bm{M}$, i.e. the approximate set of linear regression models and subsets of $\bm{DS}$ fitting them}
	$\bm{M} \leftarrow \emptyset $, $m \leftarrow 0$, $S \leftarrow \bm{DS}$ \;
	$f \leftarrow $ the linear regression model of $S$, $p_F \leftarrow $ the F-test's p-value of $f$ on $S$ \;
	\eIf{$p_F < 0.05$}{
		$\bm{M} \leftarrow \{ (f,S) \}$ \; 		
		\Return $\bm{M}$ \;
	}{
		$n \leftarrow |\bm{DS}|$, $D \leftarrow \bm{DS}$, $\sigma^2 \leftarrow estimate(\bm{DS})$ \;
		$n_s \leftarrow (k+1)(\frac{1}{\epsilon} \Phi^{-1}(\frac{2-\delta}{2})\sigma)^{2/3}$, $L \leftarrow \frac{4\sqrt{3}\sigma}{\epsilon \sqrt{\max \{65,3k\}}\delta}$ \;
		\textbf{While} ({$n > N_0$ and $m < M_0 - 1$})
		\DoWhile{
			$S \leftarrow \mathrm{Subset}(\bm{DS},n_s,L)$, $D \leftarrow D-S$\;
			$f \leftarrow $ linear regression model of $S$, $p_F \leftarrow $ the F-test's p-value of $f$ on $S$ \;
			\textbf{While} ({$p_F \geq 0.05$ and $D \neq \emptyset $})
			\DoWhile{
				$S \leftarrow \mathrm{Subset}(\bm{DS},n_s,L)$\;
				$f \leftarrow $ linear regression model of $S$, $p_F \leftarrow $ the F-test's p-value of $f$ on $S$ \;
			}
			$\tilde{\sigma}_f^2=\frac{(\bm{y}-\bm{X}\hat{\bm{\beta}})'(\bm{y}-\bm{X}\hat{\bm{\beta}})}{n}$, where $\bm{\beta}$ is the coefficients of $f$, $\bm{y}$ and $\bm{X}$ are response variables and data matrix of $S$ respectively, $b_f \leftarrow 3 \tilde{\sigma}_f$ \;
			\For{each $d=(\bm{x},y)$ in $\bm{DS}-S$}{
				\If{$|f(\bm{x})-y| \leq b_f$}{
					$S \leftarrow S \cup \{ d \}$ \;		
				}
			}
			$\bm{M} \leftarrow \bm{M} \cup \{(f,S)\}$, $\bm{DS} \leftarrow \bm{DS} - S$, $n \leftarrow |\bm{DS}|, m \leftarrow m+1$ \;	
		}
		$S \leftarrow \bm{DS}$, $f \leftarrow $ linear regression model of $S$ \;
		$\bm{M} \leftarrow \bm{M} \cup \{(f,S)\}$ \;
	}
	\Return $\bm{M}$\;
\end{algorithm}

\begin{algorithm}[!htb]
	\label{a_approximation}
	\SetKwBlock{DoWhile}{Do}{end}
	\caption{Sample($\bm{DS},n_s,L$)}
	\KwIn{A $k$-dimensional dataset $\bm{DS}$ with $N$ data points, smallest sample size $n_s$, value range $L$;}
	\KwOut{subset $S\subset \bm{DS}$}
	$D \leftarrow$ subset of $\bm{DS}$ that every $(\bm{x},y)$ satisfies $\min x_i + \frac{L}{2} \leq d_i \leq \max x_i - \frac{L}{2},i=1,2,\cdots,k$ \;
	$d(d_1,\cdots,d_k,y_d) \leftarrow $ randomly choose a data point from $D$ \;
	$H \leftarrow [d_1 - \frac{L}{2},d_1+\frac{L}{2}] \times \cdots \times [d_k - \frac{L}{2},d_k+\frac{L}{2}]$ , $S \leftarrow \bm{DS}\cap H$ \;
	\uIf{$|S| \geq n_s$}{
		$S \leftarrow$ uniformly randomly choose $\min \{n_s,\max\{65,3k\}\}$ from $S$ \;		
		\Return $S$ \;
	}
	\uElseIf{$|S| < \max\{65,3k\}$}{
		\textbf{While} ({$|S| < \max\{65,3k\}$})
		\DoWhile{
			$j \leftarrow $ randomly choose from $1,2,\cdots,k$ \;
			$L_j \leftarrow \frac{\max\{65,3k\}}{|S|} L$ \;
			\eIf{ $d_j + \frac{L_j}{2} > \max x_j$ }{
				$H_j \leftarrow [d_j - \frac{L_j}{2},d_j+\frac{L}{2}]$ \;}{
				$H_j \leftarrow [d_j - \frac{L}{2},d_j+\frac{L_j}{2}]$ \;}
			$S \leftarrow \bm{DS}\cap H$ \;
		}
	}
	\uElse{
		\Return $S$ \;
	}
\end{algorithm}

Line 1-8 are the pre-processing phase of MMLR. Line 1-5 construct a linear regression model on the whole dataset $\bm{DS}$. If the regression model is precise enough, there's no need to use multiple models fitting $\bm{DS}$. MMLR uses the index $p_F$ to determine whether one linear model is enough. If not, MMLR would begin to construct multiple model $\bm{M}$. In line 7-8, MMLR firstly calculate the smallest sample size using an estimate of $\sigma^2$. There are already methods to precisely get $\sigma^2$'s estimate \cite{book2021lra}. Suppose that algorithm $estimate(\bm{DS})$ take a dataset $\bm{DS}$ as input and output $\sigma^2$, MMLR can use anyone of them, which is shown in Line 7. Thus MMLR could calculate $n_s$ and $L$ for further work in Line 8.

Line 9-19 is an iteration to construct every $\hat{f}_i$ and related $S_i$. Generally, MMLR samples small subsets to construct local models then finds the data points fitting them. When every iteration ends, MMLR abandons those data points from $\bm{DS}$. The terminal condition is that when $|\bm{DS}| \leq n_0$ or current $m = M_0 - 1$. At this time, the data points left would be marked as $S_{m+1}$.

Line 10-14 is pre-modelling phase. In this part, MMLR firstly prudently chooses a small area of the whole value range and samples from the data points in this area invoking $\mathrm{Subset}(\bm{DS},n_s,L)$ in Line 10, 13. $S$ is the sampled subset, and $D$ denotes the rest part. After that MMLR construct a regression model $f$ and get its statistical characteristic. By \emph{Theorem 2} and \emph{4}, $f$ is a highly precise model, of which $p_F$ should be small enough. If not, $S$ does not fit one linear model, which means MMLR should sample again in Line 12-14. 

Line 15-18 is the examine phase. In Line 15, MMLR firstly gives the fitting bound $b_f$ of $f$. The fitting bound means that the max prediction error of a data point $(\bm{x},y)$ if it fits $f$. So MMLR can figure out $(\bm{x},y)$ fitting $f$ if $|f(\bm{x})-y| \leq b_f$. If a data point $(\bm{x},y)$ fit $f$, $y-f(\bm{x}) \sim N(0,\sigma_f^2)$. As shown in \cite{book2021lra}, $\tilde{\sigma}_f^2=\frac{(\bm{y}-\bm{X}\hat{\bm{\beta}})'(\bm{y}-\bm{X}\hat{\bm{\beta}})}{n}$ is a unbiased estimator of $\sigma_f^2$. According to the characteristics of normal distribution, MMLR chooses $3 \tilde{\sigma}_f$ as the fitting bound of $f$, since $\Pr \{|\xi - 0| \leq 3 \sigma\}\approx 0.99$ when $\xi \sim N(0,\sigma^2)$. MMLR tests all data points that are not assigned into existing model's acting scope by checking whether $f(\bm{x})-y \leq 3 \tilde{\sigma_f}$. 
After getting all data points that belongs to this one linear model, MMLR updates $\bm{M}$, and $\bm{DS}$ by deleting those points in Line 19.

MMLR iteratively carries out the pre-modelling phase and examine phase until $|\bm{DS}|$ is small enough. When $|\bm{DS}| < N_0$ or the number of models $m=M_0-1$, the iteration stops. The current $\bm{DS}$ and the linear regression model of it will be settled as $(f_m,S_m)$ and added into $\bm{M}$. So far, MMLR gets a solution of the Approximately Optimal MMLR Problem. 

\subsection{Analysis}
By the steps of \emph{Algorithm 2}, \emph{Theorem 2} and \emph{4}, the correctness of \emph{Algorithm 2} is given as the following theorem.

\begin{theorem}
	The $\bm{M} = \{(\hat{f}_i,S_i)|1\leq i \leq m\}$ constructed by \emph{Algorithm 2} satisfies $ \mathrm{Pr} \{ \max\limits_i |\hat{f}_i - \hat{g}_i| \geq \epsilon \} \leq \delta$, where $\bm{M}_{opt} = \{(\hat{g}_i,S_i) | 1 \leq i \leq m\}$.
\end{theorem}

At last, when the input dataset satisfies some universal assumptions, the following theorem shows the time complexity of MMLR Algorithm. The proof is shown in Appendix F.

\begin{theorem}[Time complexity of MMLR]
	Suppose that $\bm{DS}$ is uniformly distributed in a big enough value range, the value range of $\bm{DS}$ could be divided into $m \leq M_0$ continuous areas $A_1,\cdots,A_m$, $S_i = \bm{DS} \cap A_i$, $|S_i|\geq n_0$ and $S_i$ can be fitted by a linear function, then the expected time complexity of \emph{Algorithm 2} is $O(M_0(N+k^3+\frac{k^2}{\epsilon^2}))$.
\end{theorem}
Such assumptions are common in low dimension situations\cite{arumugam2003emprr,diakonikolas2020sr,lokshtanov2021ppoly}, such as $k\leq 7$. Besides, for many datasets, one can control the value of $\sigma$ and $L$ under prudent normalization of $\bm{DS}$, so as to make $\bm{DS}$ satisfy the assumptions. Several experiment results on both synthetic and real-world datasets are shown in Appendix G.

\section{Conclusion and Future Work}
This paper introduces a new data analysis method using multiple-model linear regression, called MMLR. This paper gives the approximate MMLR algorithm and related mathematical proofs. MMLR has the advantages of high interpretability, high predicting precision and high efficiency of model constructing. It can deal with \emph{DPRVR} of big datasets, and the expected time complexity under some universal assumptions is $O(M_0(N+k^3+\frac{k^2}{\epsilon^2}))$, which is lower than the existing segmented regression methods.

However, there are still challenges and future work of multiple-model regression.
Firstly, the linear model could be replaced by any parametric models. Since they also have high interpretability and low time cost to construct. The conclusions of smallest sample size and measures should be calculated in another way, which is a challenge of mathematical reasoning.
Secondly, a more versatile algorithm of choosing subsets is required. Since several datasets might not satisfy the assumption of \emph{Theorem 6}, and normalization is not enough to make the Algorithms work efficiently, a more flexibly sampling method might help.
Lastly, when the dimension of $\bm{DS}$ is too high($k>10$), MMLR algorithm is not suitable since there's always no enough data points in $H$. Some dimensional reduction methods might mitigate the problem.

In conclusion,multiple model regression methodology has the potential to make great contribution to data analysis, and need more attention on its corresponding problems.

%
%
%
\bibliographystyle{splncs04}
\bibliography{reference}
%

\clearpage
\section*{Appendix}
\appendix

\section{Illustrations of DPRVR}
The following Figure 1 and 2 give an illustration and a real world example of DPRVR.
 
As shown in Figure 1, given a 3-dimensional $\bm{DS}=\{(y, x_1 , x_2)\} \subset [0, 8] \times [0, 8]$, we can use three linear regression models, $y=x_1+x_2$ for $R_1= \{ (x_1,x_2) | x_1^2+x_2^2 \leq 4  \}$, $y=x_1-x_2$ for $R_2 = [4,8] \times [4,6]$ and $y=0$ for $R_3 = [0, 8] \times [0, 8]- R_1 - R_2 $, to accurately model $\bm{DS}$. Since the three models are very different, it is impossible to model $\bm{DS}$ using any single model as accurate as the three models do.   
Thus, it's obvious that using multiple regression models to model a big dataset is more accurate than using one linear model in general.

\begin{figure}
	\centering
	\includegraphics[width=0.4\linewidth]{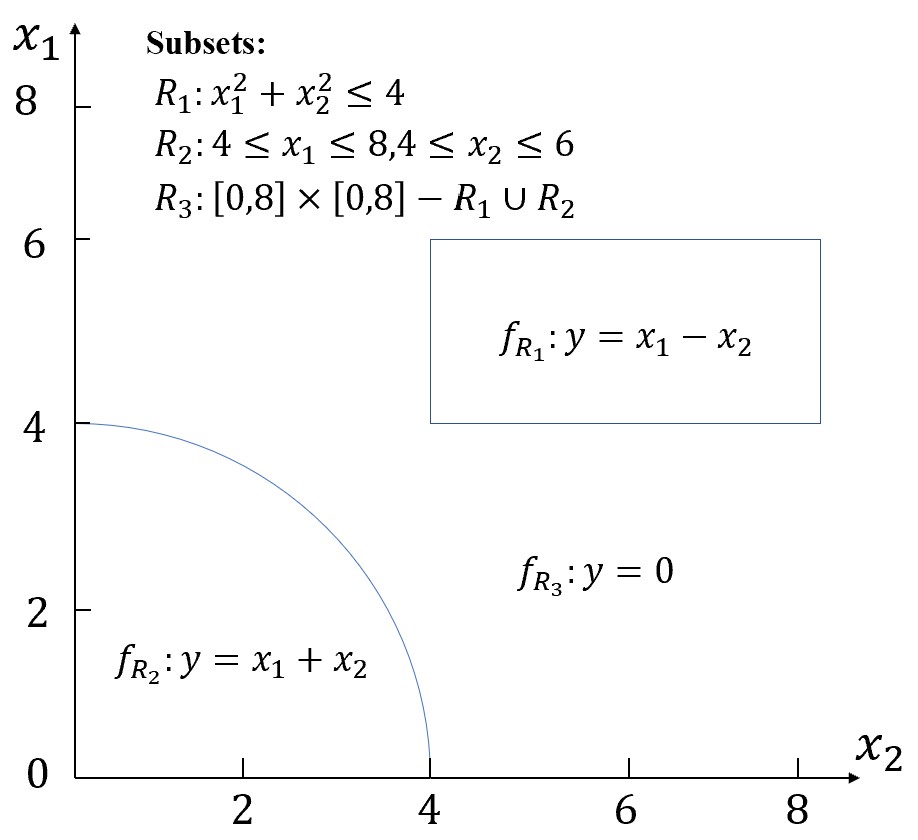}
	\caption{An illustration of diverse predictor-response variable relationships. Modelling response variable $y$ with explain variables $x_1$ and $x_2$, three local models with three regimes used to best fit the dataset.}
	\label{fig:DPRV}
\end{figure}

Figure 2 is an example of TBI dataset, given in \cite{TBI2005}, which is used to predict traumatic brain injury patients' response with sixteen explanatory variables $x_i , 1 \leq i \leq 16$. In \cite{dong2015pattern}, researchers analyzed this data with both one linear model and several. As shown in Figure 2, the $RMSE$ is 10.45 when one linear regression model is used to model the whole TBI. However the $RMSE$ is reduced to 3.51 while TBI is divided into 7 subsets and 7 different linear regression models are used to model the 7 subsets individually. 
The \emph{goodness of fit}($R^2$ for short) of the one linear regression model for modelling TBI is 0.29. But when using 7 linear regression models to model TBI, $R^2$ increases to $0.85$. 

\begin{figure}[htbp]
	\centering
	\includegraphics[width=0.8\linewidth]{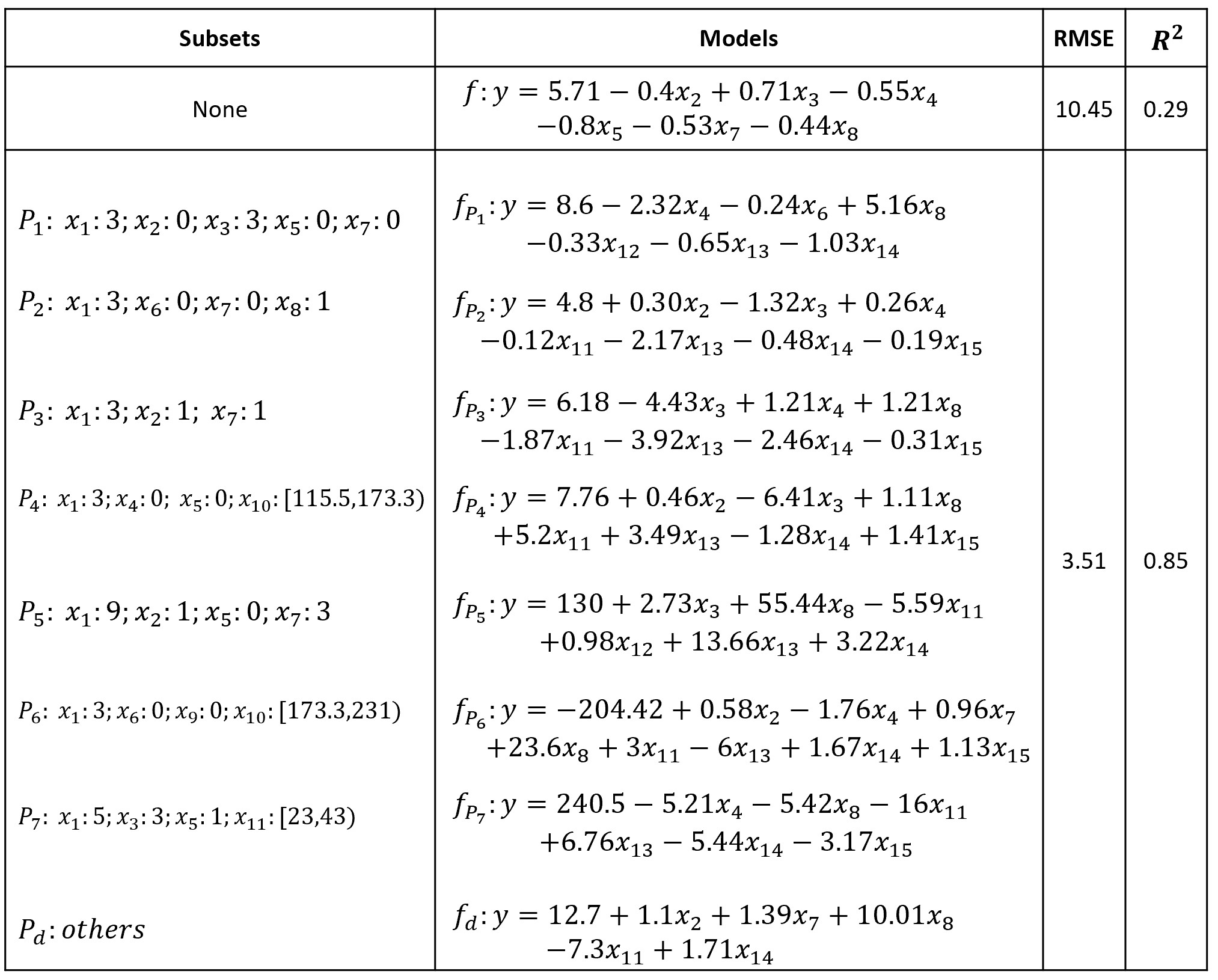}
	\caption{Using one linear regression model and 7 regression models for TBI data, has totally different prediction accuracy.}
	\label{fig:DPRV}
\end{figure}

\section{The Proof of Theorem 1}
\textbf{Theorem 1.}
	\emph{Let $\hat{\bm{\beta}}=(\hat{\beta}_0,\hat{\beta}_1,\cdots,\hat{\beta}_j)$ be constructed by the procedure of Cons-$\hat{f}$. If $\Phi(\frac{\epsilon \sqrt{t}}{\nu}) \geq \frac{2-\delta}{2}$, then $$\mathrm{Pr}\Big\{\max \limits_j  \lvert \hat{\beta}_j -\beta_j \rvert \ge \epsilon \Big\} \le \delta ,$$for any $\epsilon >0$ and $0 \leq \delta \leq 1$, where $\Phi(x)$ is the distribution function of standard normal distribution, and $\nu$ is the biggest standard deviation of all $\hat{\beta}_j$s.}
\begin{proof}
	From the construction of $\hat{\bm{\beta}}$, $$\hat{\bm{\beta}} =\frac{1}{t} \sum_{i=1}^t \hat{\bm{\beta}}^{(i)} ,$$ By \emph{Lemma 2}, we have $\mathbb{E}(\hat{\bm{\beta}})=\bm{\beta}$ and $$\mathrm{Var}(\hat{\bm{\beta}})= \frac{\sigma^2}{t^2} \sum_{i=1}^t (\bm{X}_i' \bm{X}_i)^{-1}.$$
	
	Notice that $\sum\limits_{i=1}^t (\bm{X}_i' \bm{X})_i^{-1}$ is a $(k+1)\times (k+1)$ matrix. Let $e_{pq}$ be the element in the $p$-$th$ row and $q$-$th$ column of $\sum\limits_{i=1}^t (\bm{X}_i' \bm{X}_i)^{-1}$, the $\mathrm{Var}(\hat{\bm{\beta}})$ becomes:$$\mathrm{Var}(\hat{\bm{\beta}}) = \frac{\sigma^2}{t^2} \begin{pmatrix}
		e_{00} & e_{01} & e_{02} & \cdots & e_{0k} \\
		e_{10} & e_{11} & e_{12} & \cdots & e_{1k} \\
		\vdots & \vdots & \vdots & \vdots \\
		e_{k0} & e_{k1} & e_{k2} & \cdots & e_{kk} \\
	\end{pmatrix}
	$$
	
	Let $\hat{\bm{\beta}} = (\hat{\beta_0}, \hat{\beta_1}, \cdots , \hat{\beta_k})$, then $\mathrm{Var}(\hat{\beta_j}) = \frac{\sigma^2}{t^2} e_{jj}$. Let $\nu_j = (\mathrm{Var} (\hat{\beta}_j))^{\frac{1}{2}}$. From \emph{Lemma 2}, $\nu_j = \frac{\sigma}{t} e_{jj}$. Thus, the standardized sum of $\hat{\beta_j}$ is $\zeta = \frac{1}{\nu_j \sqrt{t}} \sum\limits_{i=1}^t (\hat{\beta_j}^{(i)} - \beta_j)$ from the textbook \cite{bertsekas2008introduction}.
	
	From Lindeberg-Levi Central-Limit Theorem, 
	\begin{equation}
		\begin{aligned}
			\mathrm{Pr}\Big\{ \Big | \frac{1}{t} \sum_{i=1}^t \hat{\beta_j}^{(i)} - \beta_j \Big | < x \Big\} &=\mathrm{Pr}\Big\{ \Big | \frac{1}{\nu_j \sqrt{t}} \sum_{i=1}^t \hat{\beta_j}^{(i)} - \beta_j \Big | < \frac{x\sqrt{t}}{\nu_j} \Big\} \\
			&\cong \Phi (\frac{x\sqrt{t}}{\nu_j}) - \Phi (-\frac{x\sqrt{t}}{\nu_j}) \\
			&=2\Phi (\frac{x\sqrt{t}}{\nu_j}) -1 .
		\end{aligned}
	\end{equation}
	Since $\hat{\beta}_j = \frac{1}{t} \sum\limits_{i=1}^t \hat{\beta_j}^{(i)}$, formula (4) shows that ${\rm Pr}\{  | \hat{\beta_j} - \beta_j  | < x \} = 2\Phi(\frac{x\sqrt{t}}{\nu_j}) - 1$. Besides, ${\rm Pr}\{  | \hat{\beta_j} - \beta_j  | \geq x \} = 1 - {\rm Pr}\{  | \hat{\beta_j} - \beta_j  | < x \}$.
	Letting $x = \epsilon$, then ${\rm Pr}\{  | \hat{\beta_j} - \beta_j  | \geq \epsilon \} = 2 - 2\Phi(\frac{x\sqrt{t}}{\nu_j})$.
	$2-2\Phi(\frac{x\sqrt{t}}{\nu_j})$ is a monotonic increasing function of $\nu_j$ since $\Phi (x)$ is a monotonic increasing function of $x$.
	Letting $\nu = \max\limits_j {\nu_j}$, then $2 - 2\Phi(\frac{\epsilon \sqrt{t}}{\nu_j}) \leq 2 - 2\Phi(\frac{\epsilon \sqrt{t}}{\nu})$. From $\Phi(\frac{\epsilon \sqrt{t}}{\nu}) \geq \frac{2-\delta}{2}$, $2 - 2\Phi(\frac{\epsilon \sqrt{t}}{\nu}) \leq \delta$. Thus, all $\beta_j$s satisfies  ${\rm Pr}\{  | \hat{\beta_j} - \beta_j  | \geq \epsilon \} = 2 - 2\Phi(\frac{\epsilon \sqrt{t}}{\nu_j}) \leq \delta$. Therefore, 
	$$\mathrm{Pr}\Big\{\max \limits_j  \lvert \hat{\beta}_j -\beta_j \rvert \ge \epsilon \Big\} \le \delta ,$$for any $\epsilon >0$ and $0 \leq \delta \leq 1$.
\qed\end{proof}

\section{The Proof of Theorem 2}
Firstly, the necessary lemma to prove \emph{Theorem 2} is given as follows.

\noindent
\textbf{Lemma 2.1(Gauss-Markov).}
\emph{Suppose that the least square esitmator of $\bm{\beta}$ is $\tilde{\bm{\beta}} = (\bm{X'}\bm{X})^{-1}\bm{X'}\bm{y}$, then $\tilde{\bm{\beta}}$ has the least variance above all unbiased estimators of $\bm{\beta}$.}

Thus the \emph{Theorem 2} can be proved.
 
\noindent
\textbf{Theorem 2.}
\emph{Suppose $PS,t,n_s,\hat{\bm{\beta}} = (\hat{\beta}_1, \cdots, \hat{\beta}_k)$ are defined as in previous part, $\tilde{\bm{\beta}} = (\tilde{\beta}_1, \cdots, \tilde{\beta}_k)$ is the least square estimator of $PS$'s $\bm{\beta}$, then if $\hat{\bm{\beta}}$ satisfies inequality (1), $\tilde{\bm{\beta}}$ also satisfies inequality (1).}
\begin{proof}
	For any $0 \leq j \leq k$, let $\xi = \hat{\beta}_j - \beta_j$ and $\eta = \tilde{\beta}_j - \beta_j$ be two random variables. Thus, $\mathbb{E}(\xi ) = \mathbb{E}(\eta ) = 0$ and $\mathrm{Var}(\xi ) \geq \mathrm{Var} ( \eta)$ from \emph{Lemma 2} and \emph{Lemma 2.1}. From the expressions of $\hat{\beta}_j$ and $\tilde{\beta}_j$, they are all linear combination of all $y_i$s where $y_i \in PS$. Since $\mathbb{E} (\hat{\beta}_j ) = \mathbb{E} ( \tilde{\beta}_j ) = \beta_j$, it's reasonable to set $\hat{\beta}_j = \bm{l}_1' \bm{y} = \beta_j + \bm{l}_1'\bm{\varepsilon}$ and $\tilde{\beta}_j = \bm{l}_2' \bm{y} = \beta_j + \bm{l}_2' \bm{\varepsilon}$, where $\bm{l}_1$ and $\bm{l}_2$ are $n$-dimensional column vectors, and $\bm{l}_1', \bm{l}_2'$ are transposes of $\bm{l}_1'$ and $\bm{l}_2'$ respectively. Then $\xi = \bm{l}_1'\bm{\varepsilon}$, $\eta = \bm{l}_2' \bm{\varepsilon}$. Since every $\varepsilon_i \sim N(0,\sigma^2)$ and all $\varepsilon_i$s are independent, $\xi \sim N(0,\sigma^2_1)$, $\eta \sim N(0, \sigma^2_2)$. And $\sigma^2_2 \leq \sigma^2_1$ by \emph{Lemma 2.1}. Therefore, $\mathrm{Pr}\{ |\eta| \leq \epsilon \} \geq \mathrm{Pr}\{ |\xi| \leq \epsilon \}$, which is $\mathrm{Pr}\{ |\tilde{\beta}_j - \beta_j | \geq \epsilon \} = 1 - \mathrm{Pr}\{ |\eta| \leq \epsilon \} \leq  1 - \mathrm{Pr}\{ |\xi| \leq \epsilon \} = \mathrm{Pr}\{ |\hat{\beta}_j - \beta_j| \geq \epsilon \}$. Then when $t$ and $n_s$ satisfies the conditions of \emph{Theorem 1}, $\mathrm{Pr}\{ |\tilde{\beta}_j - \beta_j | \geq \epsilon \} \leq \mathrm{Pr}\{ |\hat{\beta}_j - \beta_j| \geq \epsilon \} \leq \delta$.
\qed\end{proof}

\section{The Proof of Theorem 3}

\textbf{Theorem 3.}
There exists an $n_s = O(\frac{1}{\epsilon^2})$ such that $$\mathrm{Pr}\Big\{\max \limits_j  \lvert \hat{\beta}_j -\beta_j \rvert \ge \epsilon \Big\} < 10^{-6}.$$
\begin{proof}
	Suppose that $p$ is defined in the step 1 of constructing $\hat{f}$, and $t, \nu$ are defined as in \emph{Theorem 1}.
	Giving $\epsilon \geq 0$, let $n_s = \frac{28.4089 p \nu^2}{\epsilon^2}$. Then $t = \frac{n_s}{p} = \frac{28.4089 \nu^2}{\epsilon^2}$, $\frac{\epsilon \sqrt{t}}{\nu} = 5.33$.
	By checking the table of standard normal distribution function, when $x \geq 5.33$, $\Phi(x) \geq 0.9999995$.
	Therefore, $\Phi(\frac{\epsilon \sqrt{t}}{\nu}) > 0.9999995 = \frac{2- 10^{-6}}{2}$.
	By \emph{Theorem 1}, $\mathrm{Pr}\{ \max|\hat{\beta}_j = \beta_j| > \epsilon \} < \delta$ if $\Phi(\frac{\epsilon \sqrt{t}}{\nu}) > \frac{2-\delta}{2}$. Thus, $\mathrm{Pr}\{ \max|\hat{\beta}_j = \beta_j| > \epsilon \} < 10^{-6}$ when $\Phi(\frac{\epsilon \sqrt{t}}{\nu}) > \frac{2- 10^{-6}}{2}$, where $n_s = \frac{28.4089 p \nu^2}{\epsilon^2} = O(\frac{1}{\epsilon^2})$.
	Therefore, there exists an $n_s = O(\frac{1}{\epsilon^2})$ such that $$\mathrm{Pr} \{ \max_j |\hat{\beta}_j - \beta_j| \geq \epsilon \} \leq 10^{-6}.$$
\qed\end{proof}

\section{The Details of Lemma 3}

In this section it is assumed that $y_i = f(\bm{x}_i) = \beta_1 x_{i1} + \cdots + \beta_k x_{ik} + \varepsilon_i$ for each $(x_{i1}, x_{i2}, \cdots , x_{ik}, y_i) \in \bm{D}$, $|\bm{D}| = n$, $\bm{X}$ is the data matrix of $\bm{D}$, $HS = H \cap \bm{D}$, $H$ is a hyper cube, $\hat{\bm{\beta}} = (\hat{\beta}_1,\cdots,\hat{\beta}_k)$ is the least square estimator of $\bm{\beta}$ by $\hat{\bm{\beta}} = (\bm{X'}\bm{X})^{-1}\bm{X'}\bm{y}$, and data points are uniformly distributed in $H$. In order to prove \emph{Lemma 3}, several conclusions should be proved before.

\noindent
\textbf{Lemma 3.1.}
\emph{
	Suppose that $\bm{D}_{ps}$ $(1 \leq p \leq k , s > 0)$ is a dataset satisfies $w_i = \beta_1 z_{i1} + \cdots  + \beta_k z_{ik} + \varepsilon_i$ for each $(z_{i1}, z_{i2}, \cdots , z_{ik}, w_i) \in \bm{D}_{ps}$, $|\bm{D}_{ps}| = n$, $z_{ij} = x_{ij}$ for $j \neq p$, $z_{ij} = s x_{ij}$ for $j \neq p$, $Z$ and $\bm{w}$ are the data matrix and response variables of $\bm{D}_{ps}$ respectively, the least square estimators of $\bm{\beta}$ of $\bm{D}$ and $\bm{D}_{ps}$ are $\hat{\bm{\beta}} = (\bm{X'}\bm{X})^{-1}\bm{X'}\bm{y}$ and $\tilde{\bm{\beta}} = (\bm{Z'}\bm{Z})^{-1}\bm{Z'}\bm{w}$ respectively, then $\mathrm{Var}(\tilde{\beta_p}) = \frac{1}{s^2}\mathrm{Var}(\hat{\beta_p})$. 
}

\begin{proof}
	From the condition of \emph{Lemma 3.1}, $\bm{Z} = \bm{X} \bm{I}_{ps}$, where $\bm{I}_{ps}$ is the identity matrix whose element of $p$-th column and $p$-th row is $s$. Thus, $\bm{Z'}\bm{Z} = (\bm{X}\bm{I}_{ps})'(\bm{X} \bm{I}_{ps}) = \bm{I'}_{ps}\bm{X'}\bm{X} \bm{I}_{ps}$, and $(\bm{Z'}\bm{Z})^{-1} = (\bm{I'}_{ps})^{-1}(\bm{X'}\bm{X})^{-1} (\bm{I}_{ps})^{-1}$. 
	Obviously, $(\bm{I}_{ps})^{-1} = (\bm{I'}_{ps})^{-1}$ is the identity matrix whose element of $p$-th column and $p$-th row is $1/s$. Let $d_{ij}$ and $e_{ij}$ be the elements of $i$-th row and $j$-th column of $(\bm{Z'}\bm{Z})^{-1}$ and $(\bm{X'}\bm{X})^{-1}$, then
	$$
	d_{ij} = \left \{ 
	\begin{array}{lr}
		e_{ij} \quad i,j \neq p, & \\
		e_{ij}/s^2 \quad i=j=p, &\\
		e_{ij}/s \quad else.& \\
	\end{array}
	\right. 
	$$
	By \emph{Lemma 1}, $\mathrm{Var}(\hat{\beta}_p) = \sigma^2 e_{pp}$, $\mathrm{Var}(\tilde{\beta}_p) = \sigma^2 d_{pp}$. Therefore, $\mathrm{Var}(\tilde{\beta}_p) = \frac{1}{s^2}\mathrm{Var}(\hat{\beta}_p)$.

\qed\end{proof}

For convenience of analysis, the following part of this section assumes that every $\bm{x}_j$ is independently uniformly distributed in $[ - \frac{L}{2}, \frac{L}{2}]$, where $[ - \frac{L}{2}, \frac{L}{2}]$ is the $j$-th dimension of $H$.

\noindent
\textbf{Lemma 3.2.}
\emph{
	$\mathrm{Pr}\{\bm{x}'\bm{x} \geq \frac{nL^2}{24}\} \geq 1 - \frac{16}{5n}$, if $\bm{x} = (x_1,x_2,\cdots,x_n)'$, $x_i \sim U[-\frac{L}{2},\frac{L}{2}](i=1,2,\cdots,n)$, and $x_1,x_2,\cdots,x_n$ are independent.
}
\begin{proof}
	Firstly, we consider the expectation and variance of $\bm{x}'\bm{x} = x_1^2 + x_2^2 + \cdots + x_n^2$. For any $x_i^2$, $x_i \sim U[-\frac{L}{2},\frac{L}{2}]$, and then $P\{x_i^2 < z\} = P\{ -\sqrt{z} < x_i < \sqrt{z} \} = \int_{-\sqrt{z}}^{\sqrt{z}} \frac{1}{L} \mathrm{d}x = \frac{2 \sqrt{z}}{L}$. Since $0 \leq x_i^2 \leq \frac{L^2}{4}$, $\mathbb{E} x_i^2 = \int_{0}^{L^2 /4} z  \mathrm{d} \frac{2 \sqrt{z}}{L} = \frac{2}{3L} z^{\frac{3}{2}} \big|_0^{L^2/4} = \frac{L^2}{12}$. Similarly, $\mathbb{E} x_i^4 = \int_{-L/2}^{L/2} x^4 \frac{1}{L} \mathrm{d}x = \frac{1}{5L} x^5 \big|_{-L/2}^{L/2} = \frac{L^2}{80}$. So, $\mathrm{Var}(x_i^2) = \mathbb{E} (x_i^2)^2 - (\mathbb{E} x_i^2)^2 = \mathbb{E} x_i^4 - (\mathbb{E} x_i^2)^2 = \frac{L^4}{180}$. Therefore, $\mathbb{E} \bm{x'}_j\bm{x}_j = \frac{nL^2}{80}, \mathrm{Var} (\bm{x'}_j\bm{x}_j) = \frac{nL^4}{180}$, since $x_1,x_2,\cdots,x_n$ are independent.
	
	Using Chebyshev inequality, there is: 
	$$\mathrm{Pr}\{ |\bm{x'}_j\bm{x}_j - \frac{nL^2}{12}| \leq \frac{nL^2}{24} \} \geq 1 - \frac{\frac{nL^4}{180}}{(\frac{n L^2}{24})^2} = 1 - \frac{16}{5n}.$$
	Further on, $\mathrm{Pr}\{ \bm{x'}_j\bm{x}_j \geq \frac{nL^2}{12} - \frac{nL^2}{24} \} \geq \mathrm{Pr}\{ |\bm{x'}_j\bm{x}_j - \frac{nL^2}{12}| \leq \frac{nL^2}{24} \} \geq 1 - \frac{16}{5n}$. Therefore, $\mathrm{Pr}\{ \bm{x'}_j\bm{x}_j  \geq \frac{nL^2}{24} \} \geq 1 - \frac{16}{5n}$.
\qed\end{proof}

\noindent
\textbf{Lemma 3.3.}
\emph{
	If $\bm{x} = (x_1,x_2,\cdots,x_n)'$, $\bm{A}$ is an idempotent matrix such that $\mathrm{tr}(\bm{A})=n-k+1(k < n)$, then when $n \to \infty$, $\mathrm{Pr}\{ \bm{x}' \bm{A} \bm{x} \geq \frac{1}{2} \bm{x}' \bm{x} \} \to 1$. 
}
\begin{proof}
	Since $\bm{A}$ is an idempotent matrix and $\mathrm{tr}(\bm{A})=k-1$, there exists an orthogonal matrix $\bm{T}$ such that $\bm{T}' \bm{AT} = \left( \begin{smallmatrix}
		\bm{I}_{n-k+1} &  \\
		& \bm{O} 
	\end{smallmatrix} \right) $, where $\bm{I}_{n-k+1}$ is an identity matrix of order $n-k+1$. Suppose $\bm{z} = \bm{T}'\bm{x}$, $\bm{x}' \bm{A} \bm{x} = \bm{z}' \left( \begin{smallmatrix}
		\bm{I}_{n-k+1} &  \\
		& \bm{O} 
	\end{smallmatrix} \right) \bm{z} = z_1^2 + z_2^2 + \cdots + z_{n-k+1}^2$. Besides, $\bm{z}'\bm{z} = \bm{x}'\bm{T}\bm{T}'\bm{x}= \bm{x}' \bm{x}$. So, $z_{n-k+2}^2 + z_{n-k+3}^2 + \cdots + z_n^2 \geq \frac{1}{2}\bm{z}'\bm{z}$ when $\bm{x}' \bm{A} \bm{x} < \frac{1}{2} \bm{x}' \bm{x}$.
	
	From linear algebra we know that $\bm{T}'\bm{x}$ could be seen as a rotation of $\bm{x}$. Let $r^2 = \bm{x}'\bm{x}$, then $\bm{z} = \bm{T}'\bm{x}$ is a point on circle $C: a_1^2 + \cdots + a_n^2 = r^2$. By the symmetry and arbitrary of $\bm{T}$, $\bm{z}$ is uniformly distributed on circle $C$. Therefore, $\mathrm{Pr}\{z_{n-k+2}^2 + z_{n-k+3}^2 + \cdots + z_n^2 \geq \frac{1}{2} \bm{z}'\bm{z} \} = \frac{S_T}{S}$, where $S$ is the surface area of $C$ and $S_T$ is the area of $C$ satisfying $z_{n-k+2}^2 + z_{n-k+3}^2 + \cdots + z_n^2 \geq \frac{r^2}{2}$.
	
	Consider the polar transform, $z_1 = \rho \cos \varphi_1, z_2 = \rho \sin \varphi_1 \cos \varphi_2$, $\cdots$, $z_n = \rho \sin \varphi_1 \sin \varphi_2 \cdots \sin \varphi_{n-2} \cos \varphi_{n-1}$, then
	\begin{align*}
		S = V'(\rho) = &\frac{\mathrm{d}}{\mathrm{d}\rho} \int_0^r \rho^{n-1} \mathrm{d}\rho \int_0^{\pi} \sin \varphi_1^{n-2} \mathrm{d} \varphi_1 \int_0^{2\pi} \sin \varphi_2^{n-3} \mathrm{d} \varphi_2 \cdot \\
		&\cdots \int_0^{2\pi} \mathrm{d}\varphi_{n-1}.
	\end{align*}
	When $z_{n-k+2}^2 + z_{n-k+3}^2 + \cdots + z_n^2 \geq \frac{r^2}{2}$, $\sin^2 \varphi_1 \sin^2 \varphi_2 \cdots \sin^2 \varphi_{n-k+1} \geq \frac{r^2}{2 \rho^2}$.
	Let $\bm{x}_{n-k+1} = (\rho,\varphi_1,\cdots,\varphi_{n-k+1})$, 
	\begin{align*}
		S_T = & \frac{\mathrm{d}}{\mathrm{d}\rho} \int_{\Omega} \rho^{n-1} \sin \varphi_1^{n-2} \mathrm{d} \varphi_1  \sin \varphi_2^{n-3} \cdot
		\cdots \sin \varphi_{n-k+1}^{k-2} \mathrm{d} \bm{x}_{n-k+1} \cdot \\
		&  \int_0^{2\pi} \varphi_{n-k+2}^{k-3} \mathrm{d}\varphi_{n-k+2} \int_0^{2\pi} \varphi_{n-k+3}^{k-4} \mathrm{d}\varphi_{k-1} \cdots \int_0^{2\pi} \mathrm{d}\varphi_{n-1} ,
	\end{align*}
	where $\Omega$ is the area of $\sin^2 \varphi_1 \sin^2 \varphi_2 \cdots \sin^2 \varphi_{n-k+1} \geq \frac{r^2}{2 \rho^2}$.
	
	Suppose that $\Omega_0 = [0,r]\times[0,\pi]\times[0,2\pi]\times\cdots\times[0,2\pi]$, then
	\begin{align*}
		\frac{S_T}{S} &= \frac{\mathrm{d}}{\mathrm{d}\rho} \frac{\int_{\Omega} \rho^{n-1} \sin \varphi_1^{n-2}  \sin \varphi_2^{n-3}\cdots \sin \varphi_{n-k+1}^{k-2} \mathrm{d}\bm{x}_{n-k+1}}{\int_{\Omega_0} \rho^{n-1} \sin \varphi_1^{n-2} \sin \varphi_2^{n-3}\cdots \sin \varphi_{n-k+1}^{k-2} \mathrm{d}\bm{x}_{n-k+1}} \\
		&\leq  \frac{|\Omega|}{|\Omega_0|} .
	\end{align*}
	From the expression of $\Omega$, $\Omega$ is symmetric around $\sin \varphi_1=\sin \varphi_2=\cdots=\sin \varphi_{n-k+1}$ and $\sin \varphi_i$s are no bigger than $1$. When $\sin \varphi_1= \sin \varphi_2=\cdots= \sin \varphi_{n-k+1}$, $\sin \varphi_i = \sqrt[2(n-k+1)]{1/2}$. So when $n \to \infty$, $\sin \varphi_i \to 1$ if all $\sin \varphi_i$s are equal. And when $n$ is big enough, the measure of $\Omega$'s subset with a center must smaller than 1. Since there are $2^{n-k-1}$ centers of symmetry, $|\Omega| \leq 2^{n-k-1} \cdot r \cdot \pi \cdot (1)^{n-k-1} = r \pi 2^{n-k-1} $. 
	Obviously, $|\Omega_0| = r \cdot \pi \cdot (2\pi) \cdots (2\pi) = r \pi (2 \pi)^{n-k-1}$. So, $\frac{S_T}{S} \leq r\pi(\frac{1}{\pi})^{n-k-1} \to 0, n \to \infty$. Therefore, when $n \to 0$, $\mathrm{Pr}\{z_{n-k+2}^2 + z_{n-k+3}^2 + \cdots + z_n^2 \geq \frac{1}{2} \bm{z}'\bm{z} \} \to 0$, which equals to $\mathrm{Pr}\{ \bm{x}' \bm{A} \bm{x} \geq \frac{1}{2} \bm{x}' \bm{x} \} \to 1$.
\qed\end{proof}

\noindent
\textbf{Lemma 3.4.}
\emph{
	Suppose that $\bm{X}_j (j=1,2,\cdots,k)$ is a submatrix of $\bm{X}$ which drops $\bm{x}_j$,
	$H_j = [ - \frac{L}{2}, \frac{L}{2}]$, then $\mathrm{Pr}\{ \mathrm{Var}(\hat{\beta}_j) \leq \frac{48 \sigma^2}{n L^2} \} \to 1, n \to \infty$.
}

\begin{proof}
	By \emph{Lemma 1}, $\mathbb{E} \hat{\bm{\beta}} = \bm{\beta}, \mathrm{Var}(\hat{\bm{\beta}}) = \sigma^2 (\bm{X'}\bm{X})^{-1}$. 
	Letting $\bm{C} = (\bm{X'}\bm{X})^{-1}$, there is $\mathrm{Var} (\hat{\beta}_j) = \sigma^2 \bm{C}_{jj}$.
	The \cite{book2021lra} proved that $$\mathrm{Var}(\hat{\beta}_j) = \sigma^2 (\bm{x}_j' \bm{x}_j - \bm{x}_j' \bm{X}_j (\bm{X}_j' \bm{X}_j)^{-1} \bm{X}_j' \bm{x}_j )^{-1},$$ 
	and then
	$$\bm{C}_{jj} = \frac{1}{\bm{x}_j'( \bm{I} - \bm{X}_j (\bm{X}_j' \bm{X}_j)^{-1} \bm{X}_j' )\bm{x}_j } .$$
	
	Let $\bm{A} = \bm{I} - \bm{X}_j (\bm{X}_j' \bm{X}_j)^{-1} \bm{X}_j'$, then $\mathrm{tr}(A)=n-k+1$ from \cite{book2021lra}. Besides, $\bm{A}$ is a idempotent matrix, which means there exists an orthogonal matrix $\bm{T}$ so that $\bm{T}' \bm{AT} = \left( \begin{smallmatrix}
		\bm{I}_{r} &  \\
		& \bm{O} 
	\end{smallmatrix} \right) $, where $\bm{I}_{r}$ is a identity matrix of order $r$ and $rank(\bm{A}) = r$. From \cite{book2021lra} we could know that $r=n-k+1$.
	
	Let $p_1 = \mathrm{Pr}\{\bm{x}_j' \bm{A} \bm{x}_j \geq \frac{1}{2} \bm{x}_j' \bm{x}_j\}$, then by \emph{Lemma 3.3}, $p_1 \to 1, n \to \infty$. Let $p_2 = \mathrm{Pr}\{\bm{x}_j' \bm{x}_j \geq \frac{nL^2}{24} \}$, then by \emph{Lemma 3.2}, $p_2 \geq 1- \frac{16}{5n}$. Since $\bm{x}_j$ and $\bm{A}$ are independent, $\mathrm{Pr}\{ \bm{x}_j' \bm{A} \bm{x}_j \geq \frac{nL^2}{48}\} = p_1 p_2 \geq p_2 (1-\frac{16}{5n})$. So $\mathrm{Pr}\{ \bm{x}_j' \bm{A} \bm{x}_j \geq \frac{nL^2}{48}\} \to 1, n\to \infty$. 
	
	From the expression of $\bm{A}$ we know that $\mathrm{Var}(\hat{\beta}_j) = \frac{\sigma^2}{\bm{x}_j' \bm{A} \bm{x}_j}$. Therefore, $\mathrm{Pr}\{\mathrm{Var}(\hat{\beta}_j) \leq \frac{48\sigma^2}{nL^2} \} = \mathrm{Pr}\{ \bm{x}_j' \bm{A} \bm{x}_j \geq \frac{nL^2}{48}\} \to 1, n\to \infty$.
\qed\end{proof}

\noindent
\textbf{Lemma 3}
\emph{
	Suppose $\bm{X}_j, H_j, L, \lambda_1$ are defined the same way as \emph{Lemma 3.4}, given $\epsilon > 0, 0 < \delta <1$, then for any $\epsilon' >  0$, there exists an $n_0$ such that when $L \geq \frac{4 \sqrt{3} \sigma}{\epsilon \sqrt{n \delta}}$ and $n > n_0$, $\mathrm{Pr}\{ |\hat{\beta}_j - \beta_j| \geq \epsilon \} \leq \delta$ holds with possibility no less than $1-\epsilon'$. Further, $\mathbb{E}(|\hat{\beta}_j - \beta_j|)$ is monotonically decreasing at $nL^2$.
}

\begin{proof}
	From \emph{Lemma 1} and \emph{Lemma 3.4}, there is $\mathbb{E} \hat{\beta}_j = \beta_j$ and $\mathrm{Pr}\{ \mathrm{Var}(\hat{\beta}_j) \leq \frac{48 \sigma^2}{n L^2} \} \to 1$ when $n \to \infty$. Using Chebyshev inequality, $\mathrm{Pr}\{|\hat{\beta}_j - \beta_j| \geq \epsilon \} \leq \frac{\mathrm{Var}(\hat{\beta}_j)}{\epsilon^2}$.
	Since $\lim\limits_{n\to \infty} \mathrm{Pr}\{ \mathrm{Var}(\hat{\beta}_j) \leq \frac{48 \sigma^2}{n L^2} \} =1$, for any $\epsilon' >  0$, there exists an $n_0$ such that $\mathrm{Pr}\{ \mathrm{Var}(\hat{\beta}_j) \leq \frac{48 \sigma^2}{n L^2} \} > 1 - \epsilon'$ when $n > n_0$.
	Let $\frac{48\sigma^2}{\epsilon^2 n L^2} \leq \delta$, then $L \geq \frac{4\sqrt{3}\sigma}{\epsilon \sqrt{n \delta} }$.
	Therefore, when $L \geq \frac{4\sqrt{3}\sigma}{\epsilon \sqrt{n \delta} }$ and $n > n_0$, there are $\mathrm{Pr}\{\mathrm{Var}(\hat{\beta}_j) \leq \epsilon^2 \delta\} \geq 1-\epsilon'$ and $\mathrm{Pr}\{|\hat{\beta}_j - \beta_j| \geq \epsilon \} \leq \frac{\mathrm{Var}(\hat{\beta}_j)}{\epsilon^2}$, which means $\mathrm{Pr}\{ |\hat{\beta}_j - \beta_j| \geq \epsilon \} \leq \delta$ holds with possibility no less than $1-\epsilon'$.
	And when $nL^2$ increases, $\mathrm{Var}(\hat{\beta}_j)$ decreases, which means $\mathbb{E}(|\hat{\beta}_j - \beta_j|)$ is monotonically decreasing at $nL^2$.
\qed\end{proof}

\section{The Proof of Theorem 5}
\noindent
\textbf{Theorem 5}
\emph{
	Suppose that $\bm{DS}$ is uniformly distributed in a big enough value range, the value range of $\bm{DS}$ could be divided into $m \leq M_0$ continuous areas $A_1,\cdots,A_m$, $S_i = \bm{DS} \cap A_i$, $|S_i|\geq n_0$ and $S_i$ can be fitted by a linear function, then the expected time complexity of \emph{Algorithm 2} is $O(M_0(N+k^3+\frac{k^2}{\epsilon^2}))$.
}
\begin{proof}
	From the process of \emph{Algorithm 2}, the time complexity of constructing $f$ in Line 2 is $O(k^2N+k^3)$. The time complexity of estimate($\bm{DS}$) is $O(N)$. Calculating $n_s$ and $L$ costs $O(1)$ time and $n_s = O(\frac{1}{\epsilon^2})$. The provoked Subset($\bm{DS}$,$n_s$,$L$) costs $O(N)$ time. Constructing $f$ with subset costs $O(k^2n_s+k^3)$ time. The examine phase costs $O(N)$ time. 
	Consider the chances of successfully choosing an $H$, such that $H\cap \bm{DS} \subset A_j$ for some $j=1,2,\cdots,m$. Denote the possibility of choosing an $H$, such that $H\cap \bm{DS} \subset A_j$ for some $j=1,2,\cdots,m$ as $P{H \cap \bm{DS} \subset A_j} = p$. Since the value range of $\bm{DS}$ is big enough, $p>>0.001$. Therefore, the expected time $t$ of choosing $H$ is $\mathbb{E}t = \frac{1}{p}$. Then, the expected time complexity of all Subset($\bm{DS}$,$n_s$,$L$) invoked in each iteration is $O(\frac{1}{p}N)=O(N)$. The iteration will carry out at most $M_0$ times totally. Therefore, the expected time complexity of \emph{Algorithm 2} is $O(k^2N+k^3) + O(M_0(O(N)+O(k^2n_s+k^3))) = O(M_0(N+k^3+\frac{k^2}{\epsilon^2}))$.
\qed\end{proof}

\section{Experiment Results}
This section gives an experimental evaluation of MMLR. Two important aspects are considered, computation time and prediction accuracy, which uses RMSE as criterion.

\textbf{Platform:} The experiments are conducted in Rstudio as IDE, using R 4.1.2 as experimental environment, on a single processor machine, with a 11th Gen Intel(R) Core(TM) i9-11900K CPU of 3.50GHz, and 128GB of main memory.

\textbf{Dataset:} The experiments use both synthetic and real-world dataset. To best showing the performance and functions of regression, all datasets are of numeric data type, and have continuous label values.

One hundred synthetic datasets are used in the experiment. They have 1 millions data points to simulate big data environment. Each one of them has 5-20 subsets fitting different linear functions. The dimensions of them are 1-5, and each kind has 20 datasets examples.

Four real-world datasets for regression and classification are used in the experiments, which are got from the PMLB repository\cite{romano2022pmlb}. The datasets are of $2-6$ dimensions and 100000 data points.

This paper compare with several commonly used regression methods, linear regression(LR), Bayesian additive regression tree(CART), gradient boosting decision tree(GBDT), random forest and model tree(MT). Since the piecewise regression and segmented regression methods used more than hours to construct models, which is not efficient in practice, the related consequences are not listed for real-world datasets. The MT and random forest also spend too much time on synthetic datasets, the related results are not listed as well.

This paper also uses 60 small datasets with size of 2000 instances to compare with more methods. The dimensions of them are 2-4 and each of them has 5-20 subsets to fit different linear models. Support vector regression(SVR), neural network with reLU(NN) are also compared for these small datasets.

The results of synthetic datasets are shown in Figure 3 and Figure 4. The values are average of all synthetic datasets. The results of real-world datasets are shown in Table 1.

From the experiment results, it is obvious MMLR has higher efficiency and higher prediction accuracy than a number of commonly used methods on many synthetic and real-world datasets. 
For small datasets, MMLR keeps high accuracy. Although random forest has less RMSE and model-construct time, the scalability to big data and interpretability is worse than MMLR. In conclusion, MMLR is a competitive method for regression tasks.

\begin{figure*}
	
	\centering
	\begin{minipage}[t]{0.5\linewidth}
		\centering
		\includegraphics[width=6.0cm]{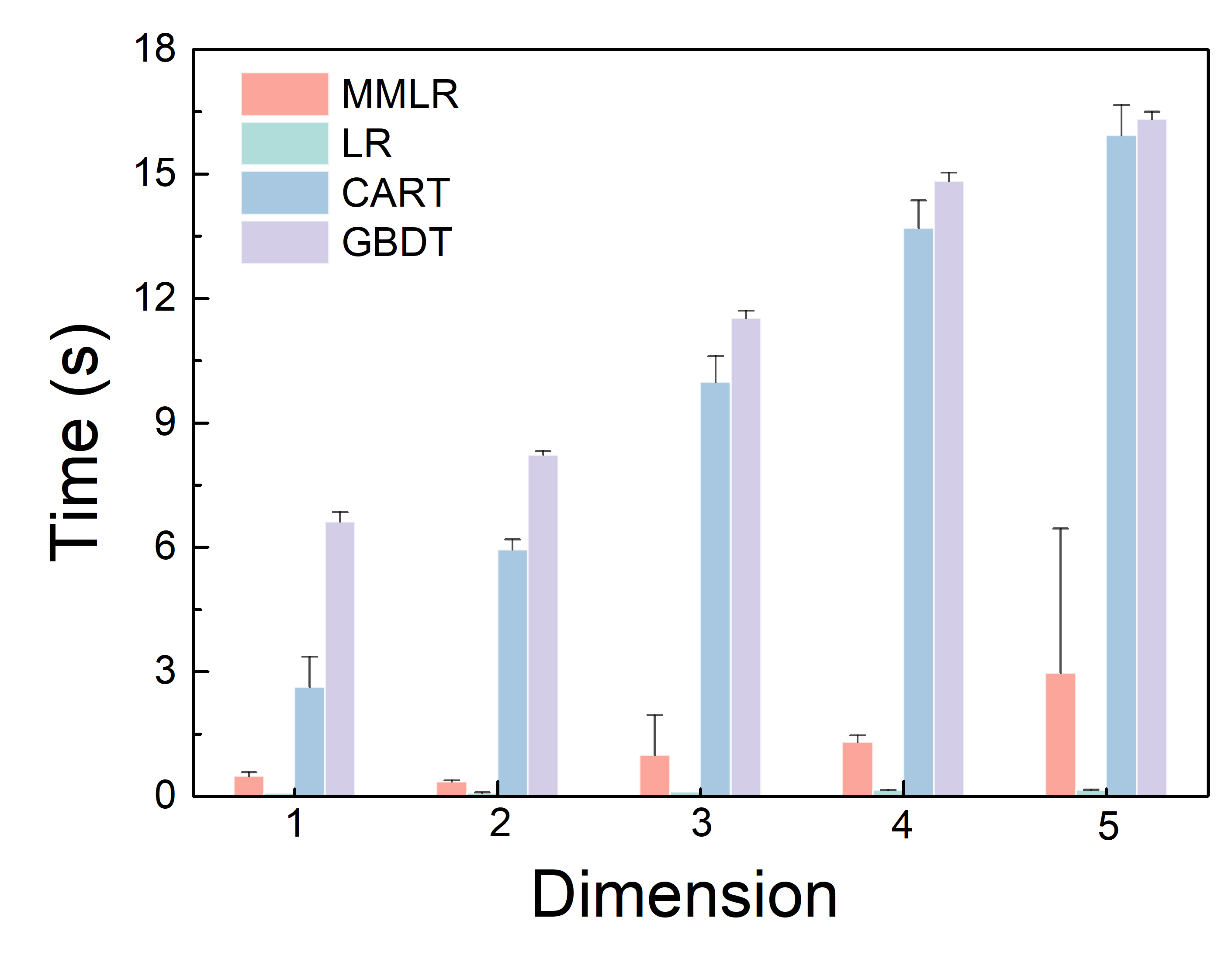}
		\caption*{(a) TIME}

	\end{minipage}%
	\begin{minipage}[t]{0.5\linewidth}
		\centering
		\includegraphics[width=6.0cm]{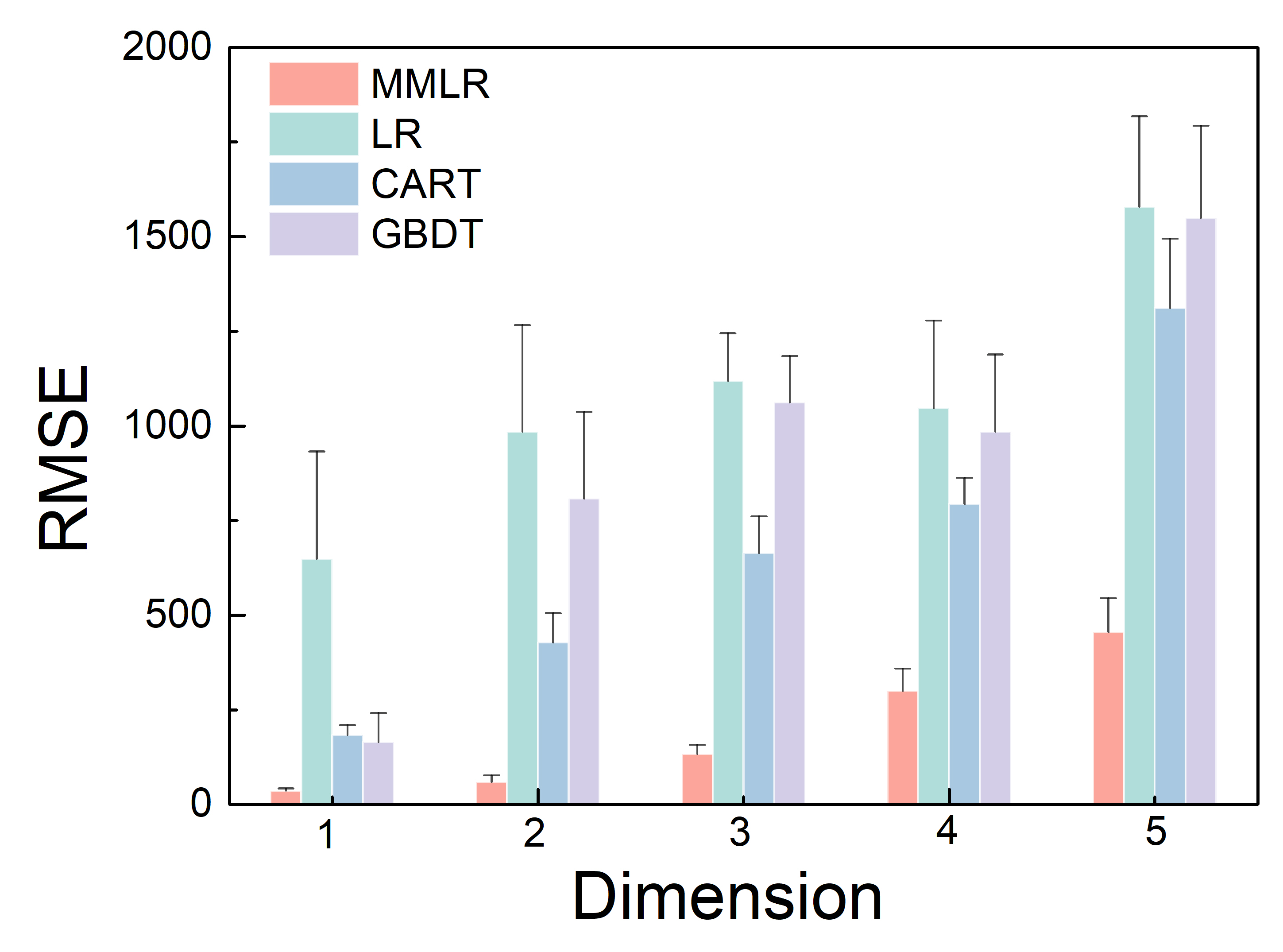}
		\caption*{(b) RMSE}
		
		\label{fig:side:a}
	\end{minipage}

	\caption{ Comparison on Time,RMSE for 4 methods on synthetic datasets}
\end{figure*}

\begin{figure*}
	
	\centering
	\begin{minipage}[t]{0.5\linewidth}
		\centering
		\includegraphics[width=6.0cm]{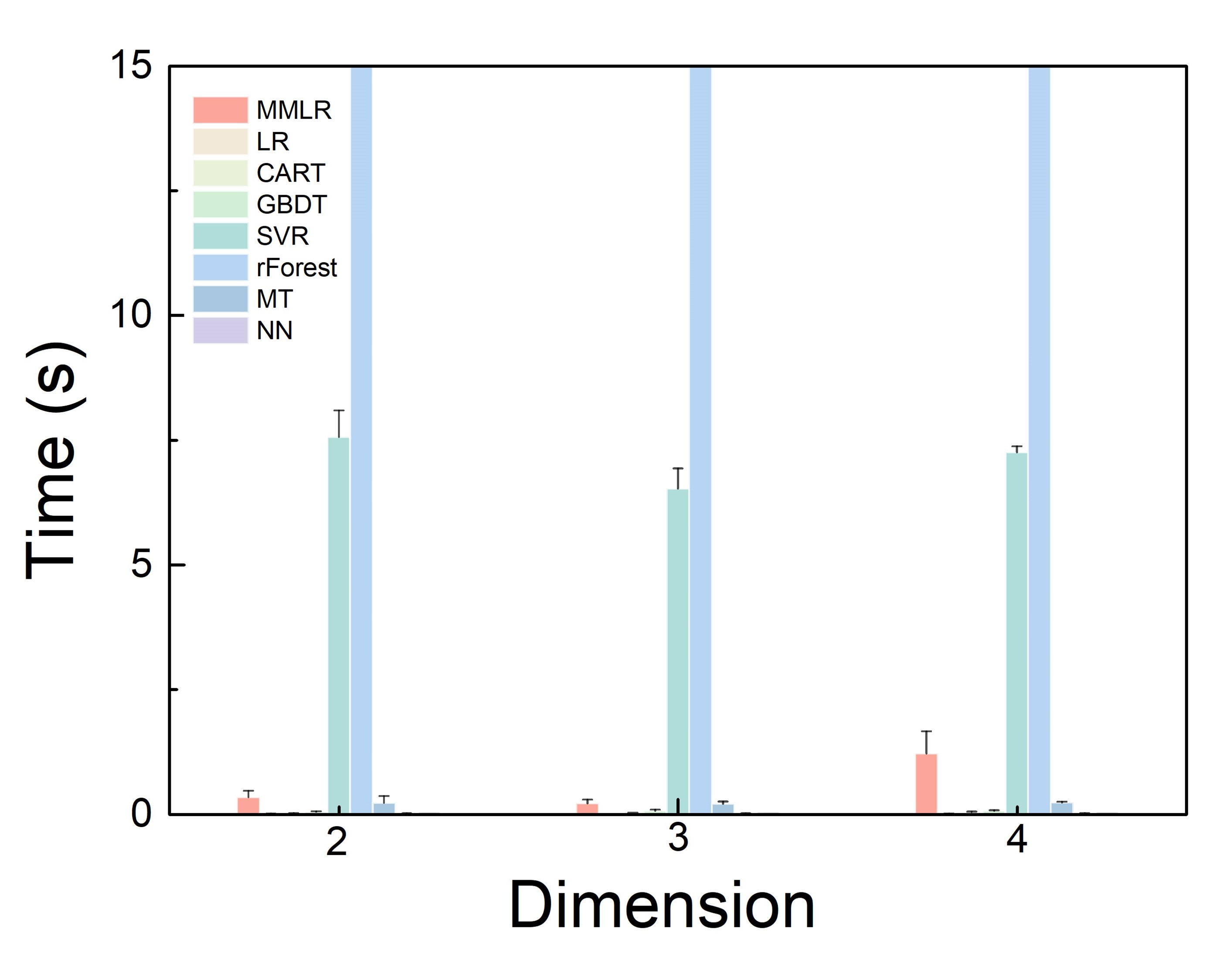}
		\caption*{(a) TIME}
		
	\end{minipage}%
	\begin{minipage}[t]{0.5\linewidth}
		\centering
		\includegraphics[width=6.0cm]{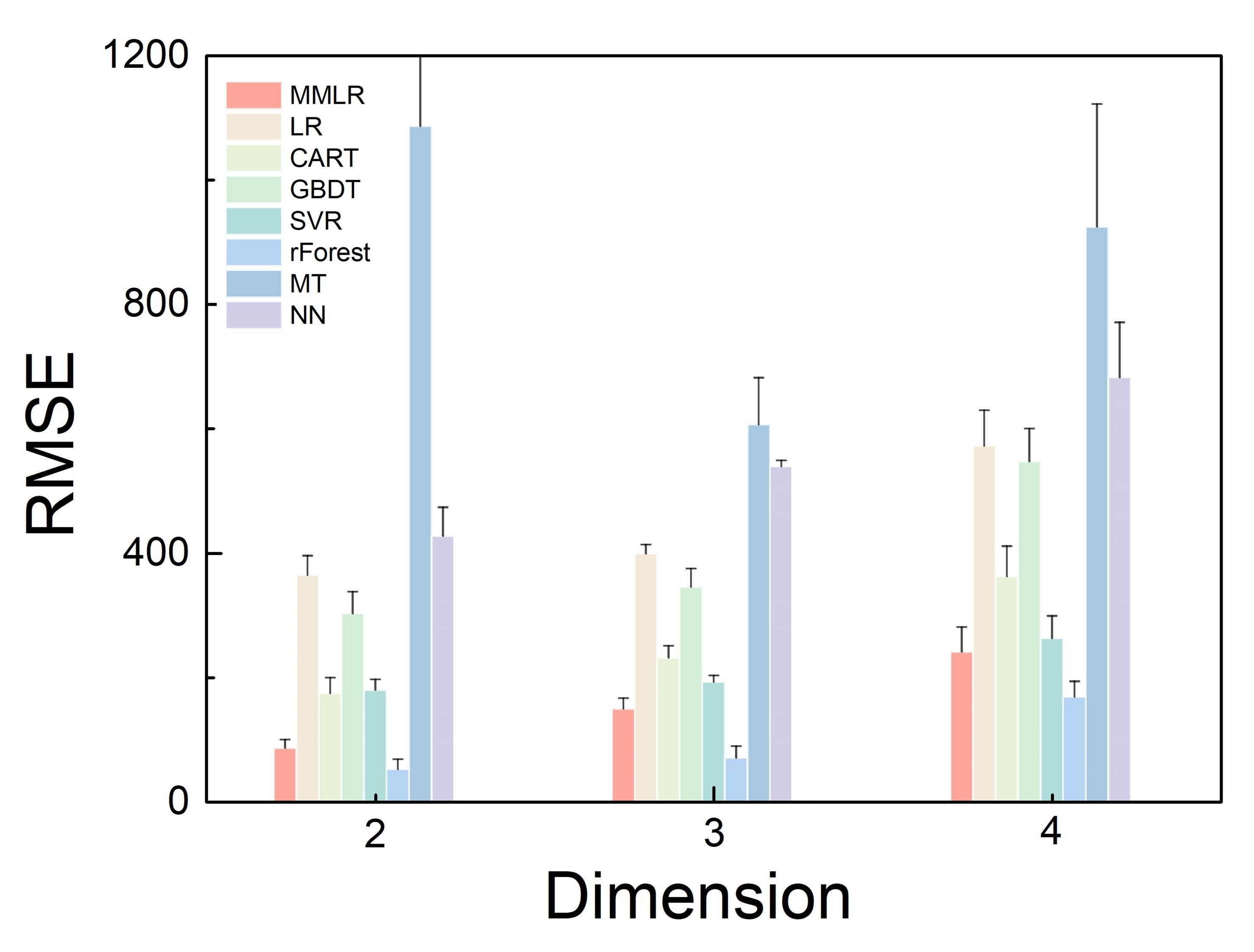}
		\caption*{(b) RMSE}
		
		\label{fig:side:a}
	\end{minipage}
	
	\caption{ Comparison on Time,RMSE for 8 methods on small datasets}
\end{figure*}

\begin{table}[tp]
	\centering
	\begin{tabular}{clccccl}
		
		\toprule
		Methods & Datasets & RMSE & MAE & TIME(s) \\
		\midrule
		MMLR & Feynman-I-12-1    & 7.089 & 5.089 & 0.47 \\
		& Feynman-test-3    & 13.524 & 9.514 & 0.39 \\
		& Feynman-I-10-7    & 3.225 & 2.584 & 0.07 \\
		& Feynman-I-11-19   & 51.77 & 41.937 & 0.56 \\
		LR  & Feynman-I-12-1    & 66.3814 & 49.871 & 0.01 \\
		& Feynman-test-3    & 28.574 & 20.782 & 0.02 \\
		& Feynman-I-10-7    & 4.474 & 2.947 & 0.01 \\
		& Feynman-I-11-19   & 116.829 & 92.522 & 0.01 \\
		CART & Feynman-I-12-1    & 82.851 & 67.836 & 0.15 \\
		& Feynman-test-3    & 36.679 & 28.464 & 0.31 \\
		& Feynman-I-10-7    & 10.406 & 8.033 & 0.19 \\
		& Feynman-I-11-19   & 279.372 & 223.848 & 0.64 \\
		GBDT & Feynman-I-12-1    & 70.677 & 53.081 & 0.41 \\
		& Feynman-test-3    & 29.404 & 20.778 & 0.55 \\
		& Feynman-I-10-7    & 4.831 & 2.962 & 0.47 \\
		& Feynman-I-11-19   & 189.053 & 149.529 & 0.75 \\
		RandomForest& Feynman-I-12-1   & 70.677 & 53.081 & 9.17 \\
		& Feynman-test-3    & * & * & * \\
		& Feynman-I-10-7    & * & * & * \\
		& Feynman-I-11-19   & * & * & * \\
		MT & Feynman-I-12-1    & 27.855 & 19.869 & 1.19 \\
		& Feynman-test-3    & 73.591 & 40.196 & 1.76 \\
		& Feynman-I-10-7    & 34.369 & 9.92 & 0.7 \\
		& Feynman-I-11-19   & 199.09 & 151.354 & 2.56 \\
		\bottomrule
	\end{tabular}
	\setlength{\belowcaptionskip}{10pt}
	\caption{Performance of Regression Methods on Real-World Datasets}
	\label{tab:freq}
\end{table}

\end{document}